\documentclass{article} 
\usepackage{iclr2026_conference,times}

\usepackage[utf8]{inputenc}
\usepackage[T1]{fontenc}
\usepackage{booktabs}
\usepackage{amsfonts}
\usepackage{amsmath,amssymb}
\usepackage{amsthm}

\usepackage{amsmath,amsfonts,bm}

\def\eqref#1{equation~\ref{#1}}

\def\1{\bm{1}}

\DeclareMathAlphabet{\mathsfit}{\encodingdefault}{\sfdefault}{m}{sl}
\SetMathAlphabet{\mathsfit}{bold}{\encodingdefault}{\sfdefault}{bx}{n}

\DeclareMathOperator*{\argmax}{arg\,max}

\usepackage{graphicx}
\usepackage{multirow}
\usepackage{array}
\usepackage{microtype}
\usepackage{xcolor}
\usepackage{algorithm}
\usepackage{algorithmic}

\usepackage{hyperref}
\usepackage{url}

\iclrfinalcopy

\newtheorem{theorem}{Theorem}
\newtheorem{definition}{Definition}
\newtheorem{proposition}{Proposition}
\newtheorem{remark}{Remark}

\newcommand{\Lsub}{\mathcal{L}_{\mathrm{sub}}}
\newcommand{\Dx}{D(\mathbf{x})}
\newcommand{\bx}{\mathbf{x}}
\newcommand{\cmark}{\(\checkmark\)}
\newcommand{\placeholder}[2]{\fbox{\rule{0pt}{#2}\rule{#1}{0pt}}}
\newcommand{\figwidth}{0.80}
\newcommand{\incfig}[3]{%
  \IfFileExists{#1}{\includegraphics[width=#2]{#1}}{\placeholder{#2}{#3}}%
}

\title{Depth-Entropy Guided Sampling for Training-Free LLM Reasoning}

\author{
Zibin Meng$^{1}$ \quad
Peng Xie$^{1}$ \quad
Kani Chen$^{1\thanks{Corresponding author.}}$ \\
\\
\texttt{\{zmengal, pxieaf\}@connect.ust.hk} \quad
\texttt{makchen@ust.hk}
\\[1ex]
$^{1}$The Hong Kong University of Science and Technology
}

\begin{document}

\maketitle

\begin{abstract}
Reinforcement learning (RL) has become the dominant paradigm for improving the reasoning capabilities of large language models, but it requires expensive training, curated data, and reward signals. Recent work shows that sampling from sharpened base-model distributions at test time recovers much of the RL gain, yet existing methods rely solely on output-layer likelihoods and ignore the transformer's internal forward-pass dynamics. We introduce \textbf{Depth-Entropy Guided Sampling (DEGS)}, a training-free, test-time method that exploits \emph{layer-wise entropy collapse} as an intrinsic quality signal. We observe that stronger reasoners---including RL-posttrained variants---exhibit a distinctive ``late collapse'': logit-lens--decoded entropy stays elevated until deeper layers before converging. We define a per-sequence \emph{collapse depth} $D(\bx)$ and a joint objective $\pi(\bx)\propto p(\bx)^\alpha\exp\!\bigl(\beta\, D(\bx)\bigr)$ that combines sequence likelihood with this depth-entropy structure, instantiated inside an MCMC power-sampling framework (DEGS-MCMC). Across three open-weight models and four reasoning benchmarks, this near-chance per-candidate signal compounds over the sampling trajectory into state-of-the-art training-free accuracy, with gains largest out of domain and on the harder splits---exactly where likelihood alone falls short---at single-digit-percent wall-clock overhead. DEGS narrowly trails an in-house GRPO reference on the math splits GRPO was trained for, yet surpasses it out of domain on GPQA for all three models, without any training, reward model, or labeled data.
\end{abstract}

\section{Introduction}
\label{sec:introduction}

Reinforcement learning (RL) with verifiable rewards---e.g., Group Relative Policy Optimization (GRPO)~\citep{shao2024deepseekmath}---is the dominant paradigm for improving the reasoning capabilities of large language models (LLMs), posttraining frontier models to sizeable gains on mathematical, scientific, and coding benchmarks~\citep{guo2025deepseek}. Yet these gains are costly: RL posttraining requires curated data, extensive tuning, and---most critically---a reliable reward signal, which is unavailable in many domains of practical interest~\citep{prabhudesai2025confidence}. A growing body of evidence further suggests they may not reflect fundamentally new capabilities: RL-posttrained models concentrate probability mass on traces that \emph{already} have high likelihood under the base model~\citep{he2025lifting,yue2025reinforcement}, and base models can even outperform them in the multi-shot regime due to degraded diversity~\citep{song2025exploration}. This \emph{distribution sharpening} hypothesis~\citep{shao2025spurious} casts RL mainly as a filter, redistributing pass@$k$ capability into single-shot performance rather than teaching genuinely novel reasoning.

This has inspired \emph{training-free} methods that sharpen the base distribution at inference time, part of a broader move toward scaling test-time computation~\citep{snell2024scaling,welleck2024metagen,brown2024monkeys}. \citet{karan2025reasoning} formalize this through the \emph{power distribution} $p(\bx)^\alpha$ ($\alpha > 1$), targeted by Metropolis--Hastings (MH) sampling, to match GRPO without training, data, or reward models; Scalable Power Sampling~\citep{ji2026scalable} and Power-SMC~\citep{azizi2026powersmc} improve efficiency via token-level approximation and a batch-parallel SMC scheme. Together they show base models are far more capable at single-shot reasoning than standard sampling reveals.

Despite their success, all existing training-free methods rely exclusively on the \emph{output-layer likelihood} $p(\bx)$, ignoring the structured intermediate representations of the forward pass~\citep{nostalgebraist2020logitlens,belrose2023tuned,wendler2024llamas}. This is a missed opportunity: \citet{wendler2024llamas} identify three forward-pass phases (input processing, a middle ``concept space,'' and output mapping), and \citet{tan2025bupo} find that entropy evolves from high-entropy exploration in early layers to deterministic refinement at the top.

We contribute an empirical observation bridging these lines: comparing \emph{layer-wise entropy trajectories} of base models against their RL-posttrained counterparts, stronger reasoners consistently exhibit \textbf{later entropy collapse} (Figure~\ref{fig:late-collapse})---base models lock in predictions at shallow layers, while RL-posttrained models sustain elevated internal uncertainty until much deeper. This extended ``deliberation phase'' is associated with downstream accuracy (a small but consistent effect, quantified in Section~\ref{sec:experiments}) and offers a new lens on what RL posttraining achieves internally. This motivates a simple idea: \emph{if late entropy collapse is characteristic of strong reasoning, we can use it as a training-free quality signal at test time}. We formalize it through a scalar \emph{collapse depth} $D(\bx)$ and construct a joint objective $\pi(\bx) \propto p(\bx)^\alpha \cdot \exp(\beta\, D(\bx))$ that augments the power distribution with this signal. The resulting method, \textbf{Depth-Entropy Guided Sampling (DEGS)}, integrates into the MH power-sampling framework (Algorithm~\ref{alg:degs-mcmc})---without parameter updates, reward models, or labeled data.

We keep two linked claims distinct: the \emph{motivation} (Figure~\ref{fig:late-collapse}) is a cross-stage contrast---posttrained reasoners collapse later than their base models---while the \emph{operative} claim our reweighting of base-model samples exploits is internal to a single base model: among its candidates, the correct ones collapse later than the incorrect ones, which we verify directly in Section~\ref{sec:experiments} (Figure~\ref{fig:collapse-depth-predicts-correctness}).

Our contributions are as follows:
\begin{enumerate}
  \item We identify \textbf{late entropy collapse} as an empirical signature of strong reasoning models and show that the collapse depth $D(\bx)$ is a statistically significant predictor of solution correctness across models and benchmarks (Section~\ref{sec:motivation}).
  \item We propose \textbf{DEGS}, built on a joint likelihood--depth objective $\pi(\bx)\propto p(\bx)^\alpha\exp(\beta D(\bx))$ (Section~\ref{sec:method}), in two decoders: a final-layer reranker, \emph{Best-of-$N$-Entropy}, which beats log-likelihood Best-of-$N$ in all twelve cells by modest margins, and our main method \textbf{DEGS-MCMC}, which exploits the full depth profile to improve over likelihood-only power sampling in all twelve cells by substantially larger margins (median $\approx 4$ points; Table~\ref{tab:main-results}).
  \item We characterize \textbf{how a weak per-candidate signal becomes a reliable decoding gain}: collapse depth is only a near-chance discriminator in isolation (AUC $\approx 0.6$; Section~\ref{sec:motivation}), yet by Proposition~\ref{prop:bounded-perturbation} it can only reorder near-equal-likelihood candidates, so compounding this tiebreak over the MCMC trajectory yields a gain positive in all twelve cells and largest on the harder, out-of-domain splits where likelihood is weakest.
\end{enumerate}

To our knowledge, DEGS is the first method to use the \emph{depth-wise evolution of internal uncertainty}---rather than output-layer signals alone---as a test-time selection criterion for reasoning, indicating that intermediate representations carry quality signals not reducible to sequence-level likelihood.

\section{Related Work}
\label{sec:related-work}

\paragraph{RL posttraining and distribution sharpening.}
RL with verifiable rewards (GRPO)~\citep{shao2024deepseekmath,guo2025deepseek}, building on RL from human feedback~\citep{ouyang2022training} and policy-gradient methods such as PPO~\citep{schulman2017ppo}, is now the dominant approach for improving LLM reasoning, but growing evidence attributes its gains primarily to \emph{distribution sharpening} rather than new capabilities~\citep{he2025lifting,yue2025reinforcement,song2025exploration}, motivating the training-free power-sampling line above~\citep{karan2025reasoning,ji2026scalable,azizi2026powersmc}; verifier-based reranking, process/outcome reward models, self-consistency, and pure repeated sampling~\citep{cobbe2021gsm8k,uesato2022process,lightman2024verify,wang2023selfconsistency,brown2024monkeys} are simpler alternatives. Using a quality signal \emph{inside} MH connects to QuEST~\citep{faria2024quest}, SMC steering of language-model programs~\citep{lew2023smc}, and twist-guided sequential Monte Carlo~\citep{zhao2024smc}, but those steer the chain with an \emph{external}, task-specific estimator (a learned MT metric, a trained twist), whereas DEGS reads its signal off the base model's own intermediate representations---no auxiliary model, reward, or training. All these methods rely on \emph{output-layer} signals; DEGS complements them with \emph{internal} depth-entropy dynamics.

\paragraph{Internal representations and logit lens.}
The logit lens~\citep{nostalgebraist2020logitlens}, tuned lens~\citep{belrose2023tuned}, and unifying frameworks such as Patchscopes~\citep{ghandeharioun2024patchscopes} decode intermediate hidden states into token distributions, revealing evolving internal beliefs; \citet{wendler2024llamas} show these carry information orthogonal to output tokens, and \citet{tan2025bupo} find early layers maintain high entropy while top layers converge. Closest to our use of depth is DoLa~\citep{chuang2024dola}, which contrasts logit-lens projections of a later and an earlier layer to reduce hallucination; unlike DoLa's per-token, two-layer contrast, DEGS leaves the token distribution unchanged and instead aggregates the full depth profile of entropy collapse into a sequence-level reweighting signal, used as a \emph{test-time selection signal} rather than for interpretation or RL optimization.

\paragraph{Entropy signals for reasoning.}
\citet{prabhudesai2025confidence} maximize base-model confidence via RL, \citet{zhao2025entropy} use self-entropy as an RL reward, and \citet{cui2025entropy} analyze \emph{output-policy} entropy collapse during RL training; all operate on output-layer entropy within a training framework, whereas DEGS exploits entropy \emph{across the depth dimension} entirely at \emph{test time}.

\paragraph{Confidence, calibration, and adaptive depth.}
Models' \emph{own} uncertainty is also a usable signal: large models are often well calibrated and can partly predict their own correctness~\citep{kadavath2022know}, and output-distribution measures such as semantic entropy flag unreliable generations~\citep{farquhar2024semantic}. Closest in spirit to our per-token collapse depth is \emph{confidence-based early exiting}~\citep{schuster2022calm}, which halts computation at a shallow layer once a per-token confidence criterion is met; DEGS inverts this efficiency framing---rather than exiting early to save compute, it \emph{reads how deep} the network defers commitment as a test-time quality signal, leaving the forward pass intact.

\paragraph{Notation.}
We consider an autoregressive LLM $p(\bx) = \prod_t p(x_t \mid x_{<t})$ with $L$ transformer layers and hidden states $\mathbf{h}_{l,t}$. The \emph{logit lens}~\citep{nostalgebraist2020logitlens} decodes layer $l$ via the unembedding $U$, $p_l(v \mid x_{\le t}) = \mathrm{softmax}(U\mathbf{h}_{l,t})_v$, with layer-wise entropy $H_l(t) = -\sum_v p_l(v) \log p_l(v)$; the \emph{power distribution}~\citep{karan2025reasoning} $p^\alpha(\bx) \propto p(\bx)^\alpha$ ($\alpha \ge 1$) is targeted by MH sampling. Full definitions are in Appendix~\ref{sec:preliminaries-full}.

\section{Motivating Observation: Late Entropy Collapse in Strong Reasoners}
\label{sec:motivation}

Comparing layer-wise entropy trajectories $\{H_l(t)\}_{l=1}^{L}$ across training stages (Figure~\ref{fig:late-collapse}), \textbf{base models} collapse early (entropy drops to near zero at shallow layers), while \textbf{RL-posttrained models} sustain elevated entropy significantly deeper before converging---consistent with the three-phase structure of \citet{wendler2024llamas}. In Figures~\ref{fig:late-collapse}(a)--(b), base Qwen2.5-Math-7B drops below 2.5~nats by layer~10 whereas its RL-posttrained variant stays above 5~nats until layer~15--20; the same pattern holds for DeepSeek-Math-7B (Figures~\ref{fig:late-collapse}(c)--(d)), and panel~(e) confirms higher mean collapse depth for posttrained models across all four settings. Combined with the finding that base models already contain high-quality reasoning traces~\citep{karan2025reasoning}, this motivates our hypothesis: \emph{among base-model candidates, those whose entropy collapses later are more likely to be correct}.

This is the cross-stage \emph{motivation}; the base-internal claim our method uses is verified in Section~\ref{sec:experiments} (Figure~\ref{fig:collapse-depth-predicts-correctness}). Appendix~\ref{app:additional} (Figure~\ref{fig:appendix-entropy-grid}) extends the contrast to all twelve overlays, including GSM8K~\citep{cobbe2021gsm8k} and Qwen2.5-Math-72B (motivation only).

\section{Method: Depth-Entropy Guided Sampling}
\label{sec:method}

We formalize the observation into a test-time algorithm: we define \emph{collapse depth}, construct the DEGS objective, and present a lightweight reranking variant alongside the MCMC variant we ship.

\subsection{Collapse Depth}
\label{sec:collapse-depth}

\begin{figure*}[t]
\centering
\setlength{\tabcolsep}{5pt}
\begin{tabular}{cc}
\begin{minipage}[t]{0.475\textwidth}
\centering
\textbf{(a) Qwen2.5-Math-7B on MATH500}\\[2pt]
\incfig{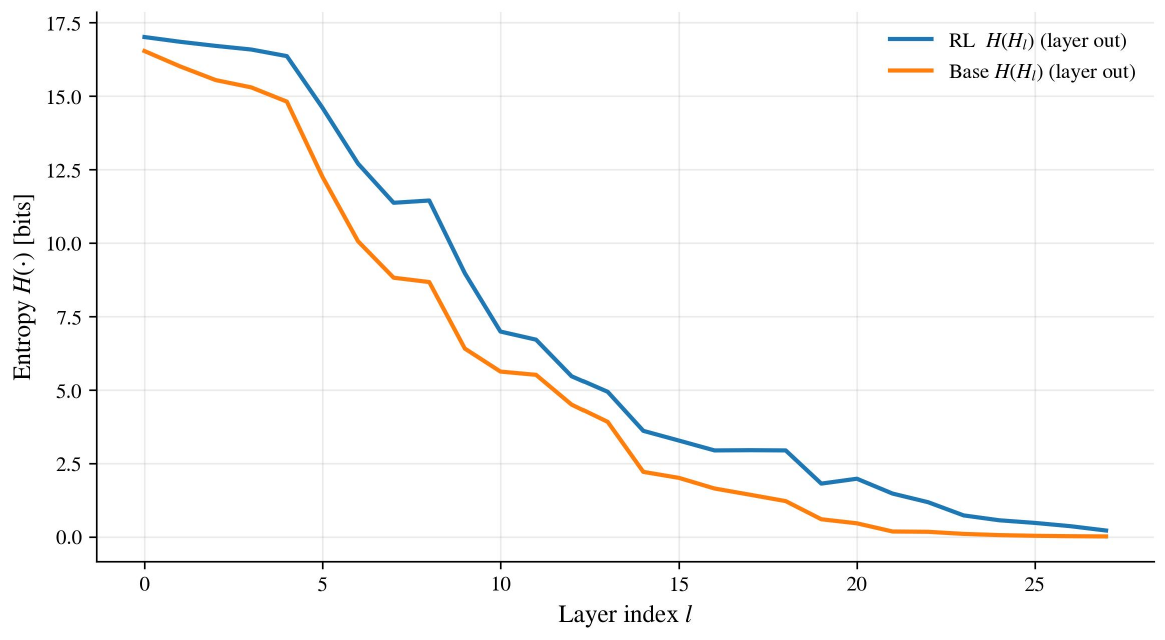}{\figwidth\linewidth}{1.55in}
\end{minipage}
&
\begin{minipage}[t]{0.475\textwidth}
\centering
\textbf{(b) Qwen2.5-Math-7B on DeepMind500}\\[2pt]
\incfig{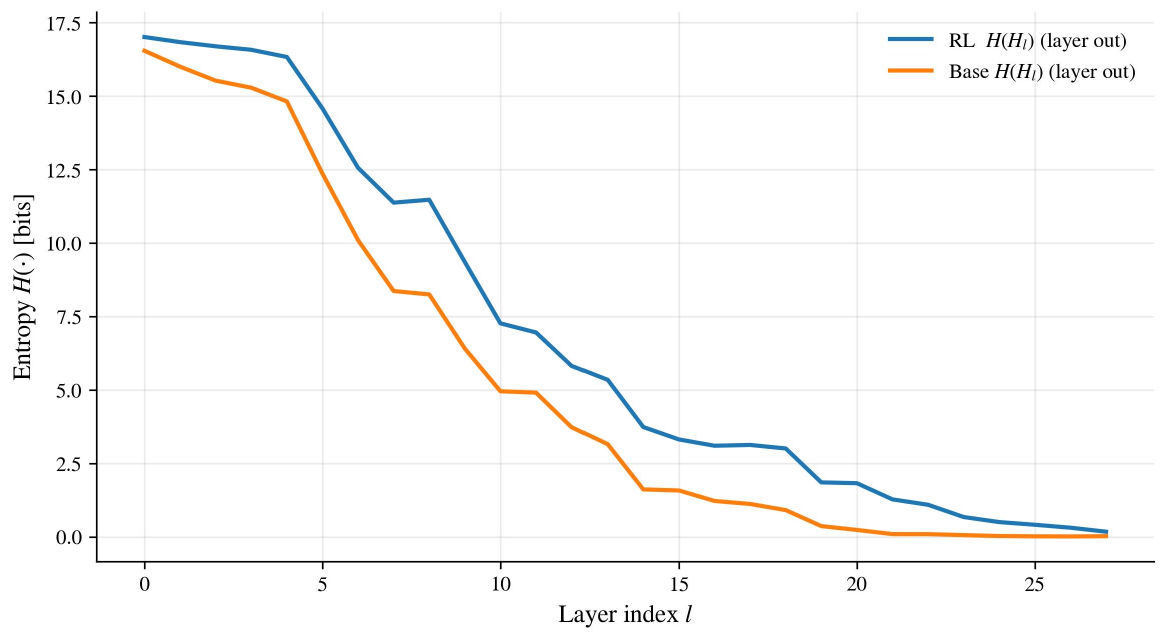}{\figwidth\linewidth}{1.55in}
\end{minipage}
\\[5pt]
\begin{minipage}[t]{0.475\textwidth}
\centering
\textbf{(c) DeepSeek-Math-7B on MATH500}\\[2pt]
\incfig{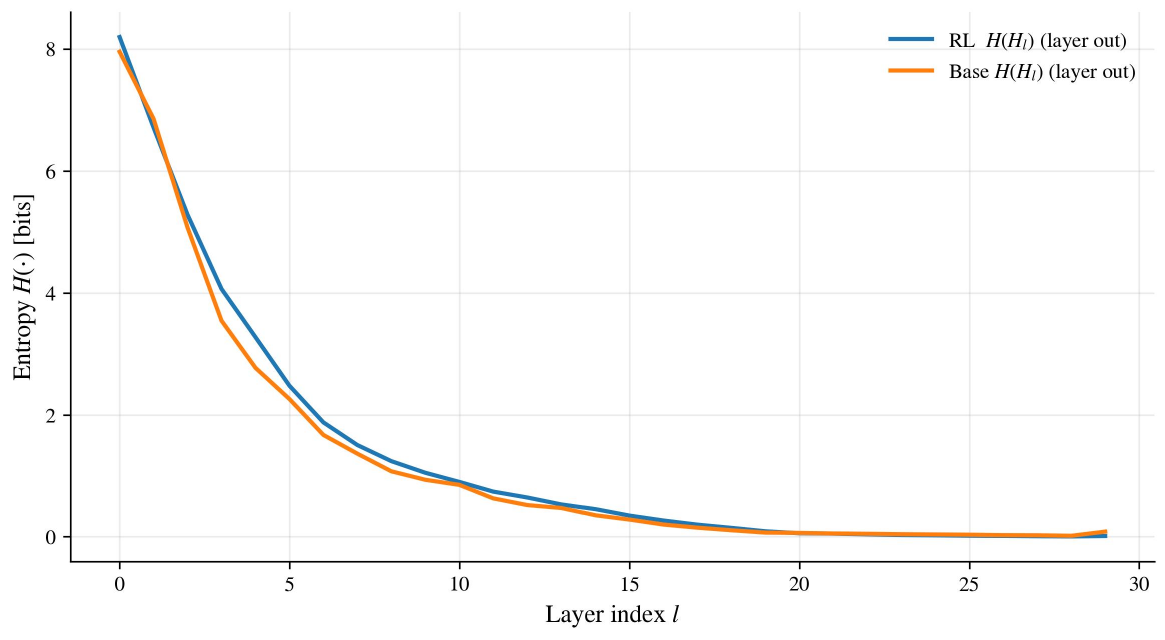}{\figwidth\linewidth}{1.55in}
\end{minipage}
&
\begin{minipage}[t]{0.475\textwidth}
\centering
\textbf{(d) DeepSeek-Math-7B on DeepMind500}\\[2pt]
\incfig{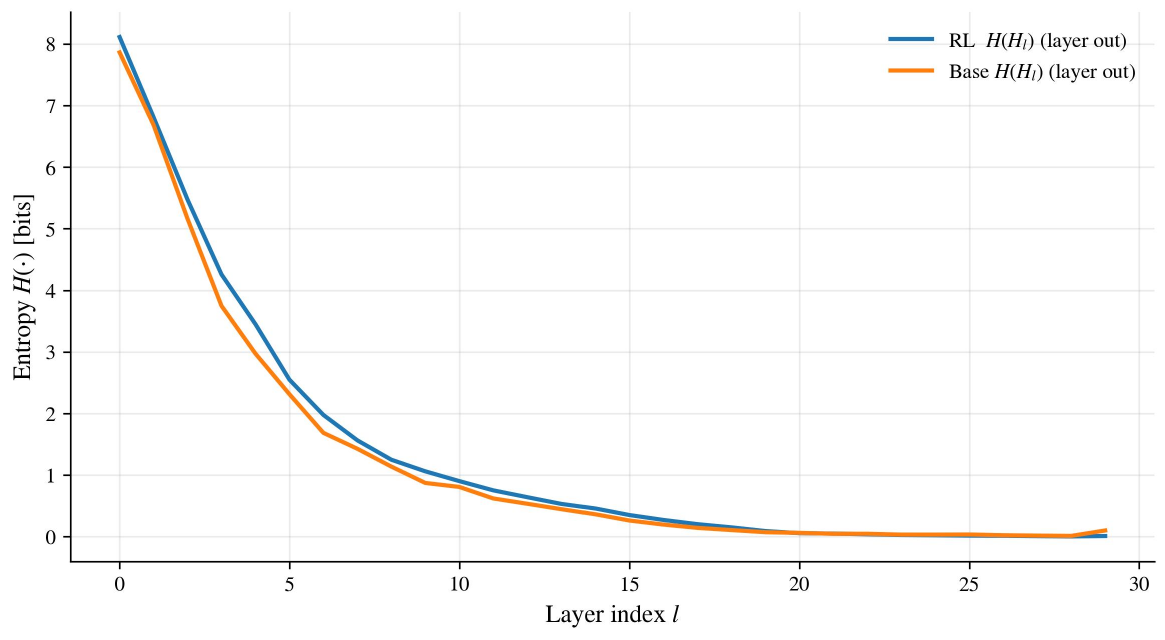}{\figwidth\linewidth}{1.55in}
\end{minipage}
\end{tabular}

\vspace{2pt}
\begin{minipage}[t]{0.72\textwidth}
\centering
\textbf{(e) Summary statistics of collapse depth}\\[2pt]
\incfig{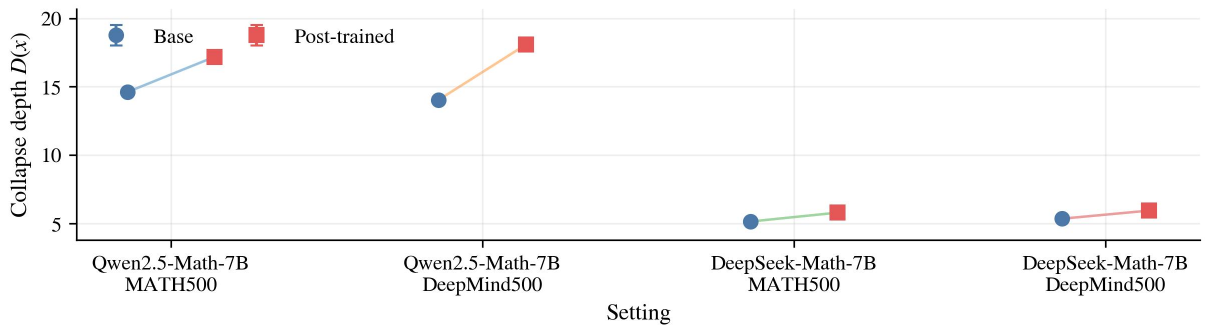}{\figwidth\linewidth}{0.90in}
\end{minipage}
\caption{Late entropy collapse in stronger reasoners. Panels (a)--(d) plot layer-wise entropy trajectories for representative base and posttrained models on mathematical benchmarks; panel (e) summarizes the effect via the mean collapse \emph{layer} $\overline{d_t}$ (average over tokens of the shallowest probed layer reaching the threshold, in absolute layer index---not the normalized $D(\bx)\in(0,1]$ of Definition~\ref{def:sequence-collapse}). Stronger reasoners collapse later in depth---equivalently, a larger $D(\bx)$---suggesting that internal uncertainty dynamics may provide a useful test-time signal.}
\label{fig:late-collapse}
\end{figure*}

\paragraph{Per-token collapse depth.}
Consider a generated sequence $\bx = (x_0, \ldots, x_T)$ from a base model $p$. To reduce cost, we compute entropy on a \emph{layer subset}
\begin{equation}\label{eq:layer-subset}
  \Lsub = \{l_1, \ldots, l_K\} \subset \{1, \ldots, L\},
\end{equation}
with $K \ll L$; by default we use evenly spaced $\Lsub = \{4, 8, 12, 16, 20, 24, 28\}$ for a 28-layer model ($K=7$).

\begin{definition}[Per-token collapse depth]\label{def:per-token-collapse}
Given an entropy threshold $\tau > 0$, the \emph{collapse depth} of token position $t$ is
\begin{equation}\label{eq:per-token-collapse}
  d_t = \min\bigl\{l \in \Lsub \;\big|\; H_l(t) \le \tau\bigr\},
\end{equation}
with the convention $d_t = L$ if no layer in $\Lsub$ satisfies the threshold.
\end{definition}

Intuitively $d_t$ measures how deep the network processes before becoming ``confident'' about token~$t$: small $d_t$ indicates (possibly premature) early commitment, large $d_t$ sustained deliberation.

\paragraph{Sequence-level collapse depth.}
\begin{definition}[Sequence-level collapse depth]\label{def:sequence-collapse}
The \emph{normalized collapse depth} of $\bx = (x_0, \ldots, x_T)$ is
\begin{equation}\label{eq:Dx}
  \Dx = \frac{1}{T+1} \sum_{t=0}^{T} \frac{d_t}{L}.
\end{equation}
\end{definition}

By construction $\Dx \in (0, 1]$; larger values indicate internal uncertainty persisting to deeper layers on average. Our central hypothesis (validated in Section~\ref{sec:experiments}) is that \emph{correct solutions tend to have larger $\Dx$}. The threshold $\tau$ (default $0.25$ nats) controls sensitivity; DEGS is robust across a range of values (Table~\ref{tab:hyperparameters-full}).

\subsection{The DEGS Objective}
\label{sec:degs-objective}

Standard power sampling~\citep{karan2025reasoning} targets $p(\bx)^\alpha$ using only the output-layer likelihood. We augment it with the depth-entropy structure captured by $\Dx$.

\begin{definition}[DEGS target distribution]\label{def:degs-target}
The DEGS target distribution over sequences is
\begin{equation}\label{eq:degs-target}
  \pi(\bx) \propto p(\bx)^\alpha \cdot \exp\!\bigl(\beta \cdot \Dx \bigr),
\end{equation}
where $\alpha \ge 1$ is the likelihood exponent and $\beta \ge 0$ the depth weight.
\end{definition}

The \textbf{likelihood term} $p(\bx)^\alpha$ favors globally coherent sequences (as in power sampling); the \textbf{depth-entropy term} $\exp(\beta\,\Dx)$ favors sequences whose internal processing resembles strong reasoners'---elevated uncertainty across layers before committing. It reduces to power sampling at $\beta = 0$ and depth-only scoring at $\alpha = 0$ (not expected to perform well alone). Equivalently, the log-score for ranking is
\begin{equation}\label{eq:degs-score}
  S(\bx) = \alpha \cdot \log p(\bx) + \beta \cdot \Dx.
\end{equation}
Defaults: $\alpha{=}4$, $\beta{=}5$, $\tau{=}0.25$ nats, $\Lsub{=}\{4,8,\ldots,28\}$, $N{=}16$ candidates, $T_{\mathrm{MCMC}}{=}10$ MH steps; full ranges in Table~\ref{tab:hyperparameters-full}.

\subsection{Reranking Variants}
\label{sec:degs-reranking}

Both variants use one scoring signal but apply it differently: reranking applies the depth tiebreak once, whereas DEGS-MCMC compounds it across the sampling trajectory. \textbf{Best-of-$N$-Entropy} ranks a fixed pool of $N$ candidates by final-layer entropy alone; it beats log-likelihood Best-of-$N$ in all twelve cells (Table~\ref{tab:main-results}), but by modest margins and well below the MCMC family. \textbf{DEGS reranking} instead uses the full depth score $S(\bx)=\alpha\log p(\bx)+\beta\Dx$ (Algorithm~\ref{alg:degs-rerank}, Appendix~\ref{app:rerank-algorithm}), computed from one teacher-forced forward that caches hidden states at $\Lsub$. By Proposition~\ref{prop:bounded-perturbation} the depth term only reorders near-equal-likelihood candidates, so a single argmax over $N{=}16$ extracts little from it; we therefore keep DEGS reranking as a controlled probe of \emph{where} in depth the signal acts (Table~\ref{tab:lsub-tradeoff}) and ship DEGS-MCMC, which compounds the same tiebreak over the trajectory at a few percent overhead relative to Best-of-$N$ (Table~\ref{tab:compute-accounting}).

\subsection{DEGS-MCMC: Integration with Power Sampling}
\label{sec:degs-mcmc}

For a more sample-efficient exploration of~\eqref{eq:degs-target}, we integrate DEGS into the MCMC power-sampling framework of \citet{karan2025reasoning}.

\paragraph{Acceptance ratio.}
In standard power sampling the target is $p(\bx)^\alpha$ with MH acceptance ratio $A_{\mathrm{std}}(\bx', \bx) = \min\bigl(1, \frac{p(\bx')^\alpha\, q(\bx \mid \bx')}{p(\bx)^\alpha\, q(\bx' \mid \bx)}\bigr)$, where $q$ is the proposal (random-index resampling). Replacing the target with $\pi(\bx)$ from~\eqref{eq:degs-target} yields
\begin{equation}\label{eq:mh-degs}
  A_{\mathrm{DEGS}}(\bx', \bx) = \min\!\left(1,\; \frac{p(\bx')^\alpha \,\exp\!\bigl(\beta\, D(\bx')\bigr)\, q(\bx \mid \bx')}{p(\bx)^\alpha \,\exp\!\bigl(\beta\, \Dx\bigr)\, q(\bx' \mid \bx)}\right).
\end{equation}
All other components---proposal, block structure, chain selection---remain identical to Algorithm~1 of \citet{karan2025reasoning}; the only addition is an entropy-probing pass to compute $D(\cdot)$.

\paragraph{Algorithmic structure.}
Following \citet{karan2025reasoning}, we divide the sequence into blocks of size $B$ with intermediate targets $\pi_k(x_{0:kB}) \propto p(x_{0:kB})^\alpha \exp\!\bigl(\beta D(x_{0:kB})\bigr)$, extending each prefix with the proposal model and then running $N_{\mathrm{MCMC}}$ MH steps with~\eqref{eq:mh-degs} (Algorithm~\ref{alg:degs-mcmc}).

\begin{algorithm}[t]
\caption{DEGS-MCMC (Entropy-Guided Power Sampling)}
\label{alg:degs-mcmc}
\begin{algorithmic}[1]
\REQUIRE Base model $p$; proposal model $p_{\mathrm{prop}}$; prompt; $\alpha, \beta, \tau$; block size $B$; max length $T$; MCMC steps $N_{\mathrm{MCMC}}$; layer subset $\Lsub$
\ENSURE Sequence $\bx_{0:T} \sim \pi$
\FOR{$k = 0$ \TO $\lceil T/B \rceil - 1$}
  \STATE \textbf{Initialize:} Extend prefix $\bx_{0:kB}$ by sampling $x_t \sim p_{\mathrm{prop}}(\cdot \mid x_{<t})$ for $t = kB{+}1, \ldots, (k{+}1)B$
  \STATE Set current state $\bx \gets \bx_{0:(k+1)B}$
  \FOR{$n = 1$ \TO $N_{\mathrm{MCMC}}$}
    \STATE Sample index $m \sim \mathrm{Uniform}\{1, \ldots, (k{+}1)B\}$
    \STATE Construct proposal $\bx'$: keep prefix $x_{0:m-1}$, resample $x'_t \sim p_{\mathrm{prop}}(\cdot \mid x'_{<t})$ for $t = m, \ldots, (k{+}1)B$
    \STATE Compute $\log p(\bx)$, $\log p(\bx')$, $D(\bx)$, $D(\bx')$, $q(\bx \mid \bx')$, $q(\bx' \mid \bx)$
    \STATE $A \gets \min\!\Big(1,\; \frac{p(\bx')^\alpha \exp(\beta\, D(\bx'))\, q(\bx \mid \bx')}{p(\bx)^\alpha \exp(\beta\, \Dx)\, q(\bx' \mid \bx)}\Big)$
    \IF{$\mathrm{Uniform}(0,1) \le A$}
      \STATE $\bx \gets \bx'$ \hfill \COMMENT{accept proposal}
    \ENDIF
  \ENDFOR
  \STATE Fix prefix $\bx_{0:(k+1)B} \gets \bx$
\ENDFOR
\RETURN $\bx_{0:T}$
\end{algorithmic}
\end{algorithm}

\paragraph{Convergence and cost.}
The chain converges to $\pi$ as $N_{\mathrm{MCMC}} \to \infty$ by standard MCMC theory~\citep{neal1993probabilistic}---$\pi(\bx)$ has nonzero mass wherever $p(\bx) > 0$ and the random-index resampling proposal is irreducible and aperiodic~\citep{karan2025reasoning} (full statement and proof in Appendix~\ref{app:convergence}). Cost is dominated by the $\frac{N_{\mathrm{MCMC}}\, T^2}{4B}$ generated tokens of power sampling, plus $O(N_{\mathrm{MCMC}} \cdot K)$ entropy computations at the probed layers (Table~\ref{tab:compute-accounting}); design choices and defaults are in Appendices~\ref{app:design-choices}--\ref{app:implementation}.

\section{Experiments}
\label{sec:experiments}

\paragraph{Setup.}
We evaluate DEGS on mathematical reasoning (MATH500~\citep{hendrycks2021math,lightman2024verify} and DeepMind500, the latter a held-out subset we construct from the DeepMind mathematics dataset~\citep{saxton2019deepmind}), code generation (HumanEval~\citep{chen2021humaneval}), and scientific reasoning (GPQA-Diamond~\citep{rein2024gpqa}) with three open-weight base models---Qwen2.5-7B~\citep{qwen2025qwen25}, Qwen2.5-Math-7B~\citep{yang2024qwen25math}, and DeepSeek-Math-7B~\citep{shao2024deepseekmath} (full benchmark, grading, and training details in Appendix~\ref{app:implementation}). The zero-shot chain-of-thought prompts~\citep{wei2022cot,kojima2022zeroshot}, answer parsers, and graders are held \emph{identical} across all decoding methods, so accuracy differences reflect the decoder rather than the harness. The \emph{GRPO (MATH)} rows are not transcribed from prior work: we train one GRPO reference per base model under a standard RLVR recipe~\citep{lambert2024tulu3} (DeepScaleR prompts~\citep{deepscaler2025}, a ground-truth math verifier, group size $16$) and evaluate it under the same protocol, so GRPO functions as a fair, in-domain reference trained under matched conditions. Because DEGS adds an entropy-probing pass, all methods are compared at matched test-time compute, normalized to forward-equivalent units (Table~\ref{tab:compute-accounting}, Section~\ref{sec:compute-budget}).

\paragraph{Baselines.}
We compare against: (1)~base decoding, (2)~low-temperature sampling, (3)~Best-of-$N$ with log-likelihood, (4)~MCMC power sampling~\citep{karan2025reasoning}, (5)~Scalable Power Sampling~\citep{ji2026scalable}, (6)~Power-SMC~\citep{azizi2026powersmc}, and (7)~GRPO-posttrained models~\citep{shao2024deepseekmath}. Our methods are \emph{Best-of-$N$-Entropy} (final-layer-entropy reranking) and \emph{DEGS-MCMC}.

\subsection{Does Collapse Depth Alone Predict Correctness?}

\begin{figure*}[t]
\centering
\setlength{\tabcolsep}{5pt}
\begin{tabular}{ccc}
\multicolumn{3}{c}{\textbf{Qwen2.5-Math-7B on MATH500}} \\[2pt]
\begin{minipage}[t]{0.315\textwidth}
\centering
\incfig{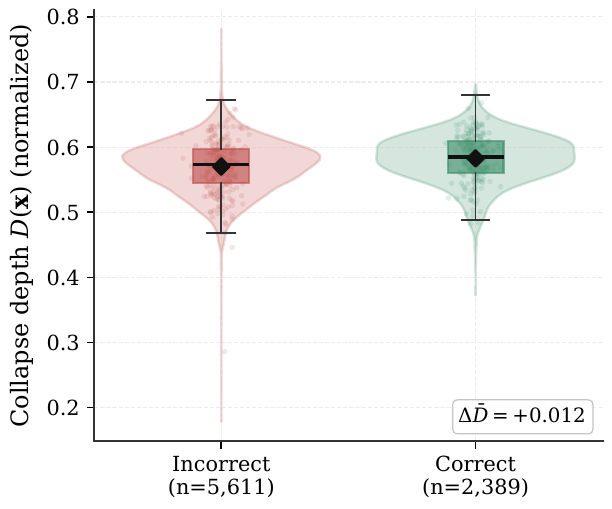}{\figwidth\linewidth}{1.30in}\\[-2pt]{\small (a)}
\end{minipage}
&
\begin{minipage}[t]{0.315\textwidth}
\centering
\incfig{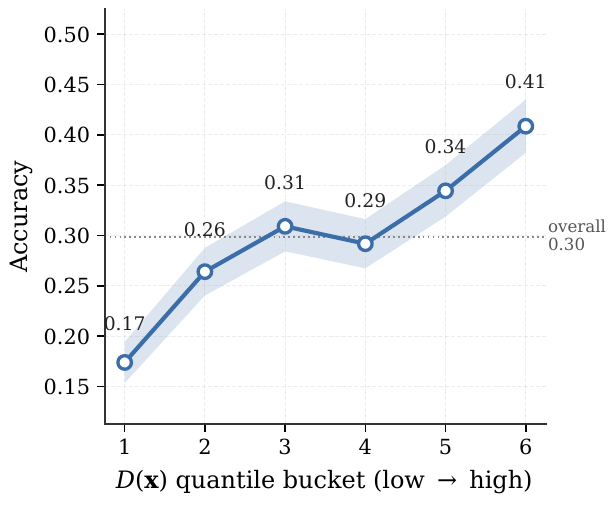}{\figwidth\linewidth}{1.30in}\\[-2pt]{\small (b)}
\end{minipage}
&
\begin{minipage}[t]{0.315\textwidth}
\centering
\incfig{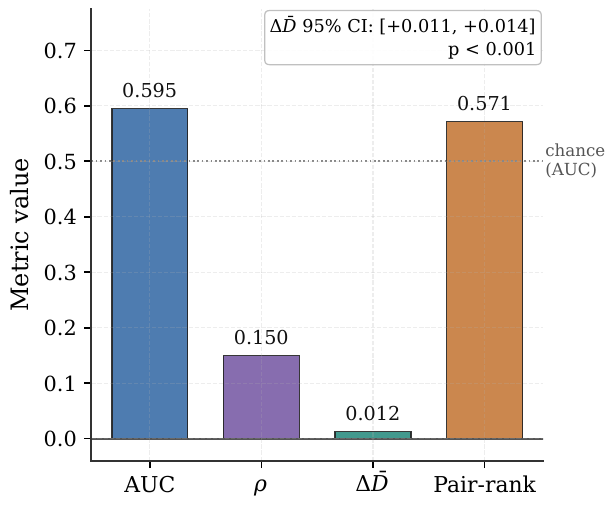}{\figwidth\linewidth}{1.30in}\\[-2pt]{\small (c)}
\end{minipage}
\\[6pt]
\multicolumn{3}{c}{\textbf{Qwen2.5-Math-7B on DeepMind500}} \\[2pt]
\begin{minipage}[t]{0.315\textwidth}
\centering
\incfig{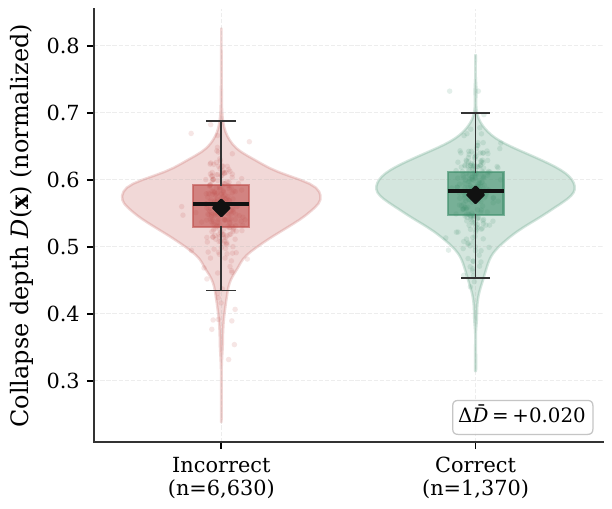}{\figwidth\linewidth}{1.30in}\\[-2pt]{\small (d)}
\end{minipage}
&
\begin{minipage}[t]{0.315\textwidth}
\centering
\incfig{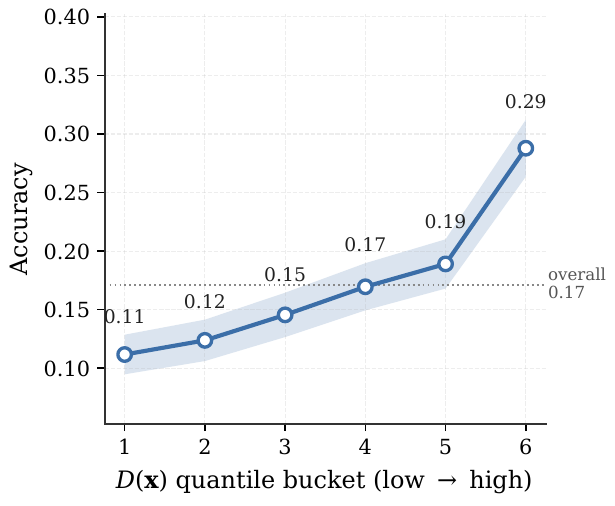}{\figwidth\linewidth}{1.30in}\\[-2pt]{\small (e)}
\end{minipage}
&
\begin{minipage}[t]{0.315\textwidth}
\centering
\incfig{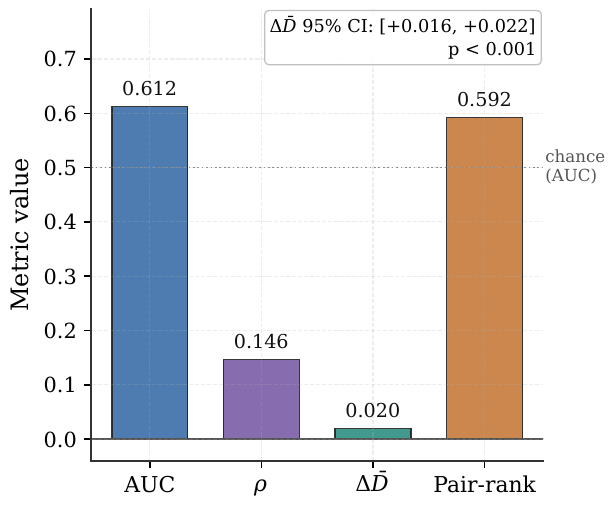}{\figwidth\linewidth}{1.30in}\\[-2pt]{\small (f)}
\end{minipage}
\\[4pt]
{\small $D(\bx)$ for correct vs.\ incorrect}
&
{\small Accuracy by $D(\bx)$ bucket}
&
{\small Predictive metrics for $D(\bx)$}
\end{tabular}
\caption{Collapse depth $D(\bx)$ as a predictor of correctness, on two representative settings: Qwen2.5-Math-7B on MATH500 (a--c) and DeepMind500 (d--f). Each row shows the distribution of $D(\bx)$ for correct vs.\ incorrect candidates, the empirical accuracy of candidates by $D(\bx)$ quantile bucket, and the predictive statistics (AUC, Spearman $\rho$, mean gap $\Delta\bar{D}$, pairwise ranking accuracy). In both settings, later collapse is positively associated with correctness.}
\label{fig:collapse-depth-predicts-correctness}
\end{figure*}

Before using $D(\bx)$ as a scoring signal, we verify it correlates with correctness: for each problem we sample $N{=}16$ candidates ($T{=}0.8$, top-$p{=}0.9$), judge correctness, and compute $D(\bx)$. On Qwen2.5-Math-7B / MATH500 (Figure~\ref{fig:collapse-depth-predicts-correctness}a--c), correct candidates have significantly higher mean collapse depth than incorrect ones ($p < 0.001$), and accuracy climbs monotonically across $D(\bx)$ buckets, from $0.17$ in the lowest to $0.41$ in the highest---a more than twofold spread that recurs, more sharply, on the harder DeepMind500 split (d--f, top bucket $\approx 2.6\times$ the bottom). Taken on its own, however, collapse depth is a \emph{weak per-candidate} discriminator: near-chance in isolation (AUC $\approx 0.6$), with a low Spearman $\rho \approx 0.15$ and pairwise-ranking accuracy barely above one-half (c,f). The significance therefore reflects the size of the $n{=}8000$ candidate pool rather than a large effect. This is precisely the regime DEGS targets: by Proposition~\ref{prop:bounded-perturbation} a moderate $\beta$ shifts each log-score by at most $\beta$ nats, so the depth term acts only as a tiebreaker among near-equal-likelihood candidates---where being right slightly more often than chance suffices once compounded over the MCMC trajectory (Section~\ref{sec:main-results}).

\subsection{Main Results}
\label{sec:main-results}

\paragraph{DEGS is the strongest training-free decoder in $11$ of $12$ cells.}
DEGS-MCMC attains the best training-free accuracy on all four benchmarks for both Qwen backbones and on three of four for DeepSeek-Math-7B. On Qwen2.5-Math-7B it improves over MCMC power sampling~\citep{karan2025reasoning} on every benchmark and beats the strongest prior training-free method, Power-SMC, on all four---though the margin narrows to $0.5$--$2$ points on the math splits, where the depth term refines rather than replaces the likelihood signal. Every power-sampling method far outstrips naive decoding (on Qwen2.5-7B / DeepMind500, $0.250\!\to\!0.573$). The lone exception is DeepSeek-Math-7B / GPQA, where low-temperature decoding is anomalously strong ($0.430$, beating every power-sampling variant including DEGS)---tellingly the setting where the late-collapse signal is weakest (Figure~\ref{fig:late-collapse}(c--e)). This is a boundary of the method: when the depth signal is uninformative, the likelihood-driven power distribution governs.

\paragraph{DEGS narrows the gap to GRPO on math and surpasses it out of domain.}
GRPO is posttrained on MATH and stays strongest in-domain, edging DEGS by $1.8$ points on MATH500 and a comparable margin on DeepMind500; even so, with no training or reward signal, DEGS closes most of the $3.7$- and $5.8$-point gaps plain power sampling leaves on those splits. Out of domain the picture reverses: DEGS exceeds GRPO on GPQA for all three models and on HumanEval for both Qwen models---reaching $0.769$ on Qwen2.5-7B HumanEval, well clear of GRPO---with GRPO keeping only a slight edge on DeepSeek-Math-7B HumanEval. This supports our framing: GRPO redistributes capability within the distribution it was optimized on, whereas DEGS reweights base-model samples and preserves cross-domain coverage. The per-cell pattern tracks the late-collapse signal: the two Qwen backbones, with the largest base-versus-posttrained collapse-depth separation (Figure~\ref{fig:late-collapse}), show the clearest gains, while DeepSeek-Math-7B---where that separation is smallest---shows almost none.

\begin{table*}[t]
\caption{Main benchmark results across models, tasks, and decoding methods, comparing standard decoding baselines, power-sampling baselines, our entropy-guided variants, and GRPO references. For DeepSeek-Math-7B, the Base / Low-temperature rows reflect the un-tuned base checkpoint under our standardized decoding protocol; with no instruction tuning the model frequently fails to emit a parsable boxed answer.}
\label{tab:main-results}
\centering
\small
\setlength{\tabcolsep}{4.5pt}
\resizebox{\textwidth}{!}{
\begin{tabular}{lcccc}
\toprule
Method & MATH500 $\uparrow$ & DeepMind500 $\uparrow$ & HumanEval Pass@1 $\uparrow$ & GPQA $\uparrow$ \\
\midrule
\multicolumn{5}{l}{\textbf{Qwen2.5-7B}} \\
Base                          & 0.498 & 0.250 & 0.329 & 0.278 \\
Low-temperature               & 0.628 & 0.342 & 0.524 & 0.303 \\
Best-of-$N$                   & 0.650 & 0.395 & 0.609 & 0.282 \\
MCMC Power Sampling           & 0.706 & 0.537 & 0.622 & 0.318 \\
Scalable Power Sampling       & 0.708 & 0.549 & 0.756 & 0.349 \\
Power-SMC                     & 0.714 & 0.563 & 0.761 & 0.351 \\
Best-of-$N$-Entropy (ours)    & \textbf{0.664} & \textbf{0.412} & \textbf{0.632} & \textbf{0.348} \\
DEGS-MCMC (ours)              & \textbf{0.728} & \textbf{0.573} & \textbf{\underline{0.769}} & \textbf{\underline{0.393}} \\
\cmidrule(lr){1-5}
GRPO (MATH)                   & \textbf{0.740} & \textbf{0.598} & \textbf{0.561} & \textbf{0.354} \\
\midrule
\multicolumn{5}{l}{\textbf{Qwen2.5-Math-7B}} \\
Base                          & 0.496 & 0.332 & 0.329 & 0.278 \\
Low-temperature               & 0.690 & 0.482 & 0.512 & 0.353 \\
Best-of-$N$                   & 0.684 & 0.465 & 0.512 & 0.343 \\
MCMC Power Sampling           & 0.748 & 0.574 & 0.573 & 0.389 \\
Scalable Power Sampling       & 0.758 & 0.593 & 0.604 & 0.409 \\
Power-SMC                     & 0.762 & 0.601 & 0.612 & 0.413 \\
Best-of-$N$-Entropy (ours)    & \textbf{0.687} & \textbf{0.479} & \textbf{0.537} & \textbf{0.369} \\
DEGS-MCMC (ours)              & \textbf{0.767} & \textbf{0.619} & \textbf{\underline{0.618}} & \textbf{\underline{0.426}} \\
\cmidrule(lr){1-5}
GRPO (MATH)                   & \textbf{0.785} & \textbf{0.632} & \textbf{0.537} & \textbf{0.399} \\
\midrule
\multicolumn{5}{l}{\textbf{DeepSeek-Math-7B}} \\
Base                          & 0.362 & 0.212 & 0.415 & 0.333 \\
Low-temperature               & 0.366 & 0.269 & 0.427 & 0.430 \\
Best-of-$N$                   & 0.420 & 0.276 & 0.433 & 0.338 \\
MCMC Power Sampling           & 0.424 & 0.307 & 0.470 & 0.345 \\
Scalable Power Sampling       & 0.464 & 0.314 & 0.487 & 0.364 \\
Power-SMC                     & 0.467 & 0.329 & 0.498 & 0.386 \\
Best-of-$N$-Entropy (ours)    & \textbf{0.423} & \textbf{0.298} & \textbf{0.453} & \textbf{0.341} \\
DEGS-MCMC (ours)              & \textbf{0.475} & \textbf{0.336} & \textbf{0.508} & \textbf{\underline{0.392}} \\
\cmidrule(lr){1-5}
GRPO (MATH)                   & \textbf{0.492} & \textbf{0.352} & \textbf{0.524} & \textbf{0.333} \\
\bottomrule
\end{tabular}
}
\end{table*}

\paragraph{Takeaway and statistical reliability.}
The depth-entropy term never lowers accuracy relative to the likelihood-only power sampling it extends: it improves on that baseline in all twelve cells and is the best training-free decoder in eleven, concentrating its gains exactly where likelihood is weakest---out-of-domain GPQA and the harder DeepMind500 split---rather than acting as a new capability. The comparison is \emph{within-framework}: toggling only $\beta$ ($5$ vs.\ $0$) holds decoder, harness, and budget fixed, so the uniformly positive per-cell deltas of Table~\ref{tab:main-results} (median $\approx 4$ points) are not confounded by the implementation differences a head-to-head with an external baseline would introduce. We claim only their \emph{uniform direction}, not any single magnitude: with $\approx 1.9$ points of binomial error on a $500$-item split we report no per-cell variance, and although the twelve-cell direction passes a one-sided sign test ($p = 2^{-12}\approx 2.4\times10^{-4}$), we read this as corroborating the $\beta$-toggle rather than as a standalone test, since the cells share models and benchmarks and are not independent.

\subsection{Accuracy versus Compute Budget}
\label{sec:compute-budget}

\paragraph{Matched-budget protocol.}
Because DEGS adds an entropy-probing pass, we normalize all methods to forward-equivalent units (Table~\ref{tab:compute-accounting}): each reranking method is charged $N$ generation passes plus probing, each MCMC-family method $B \cdot T_{\mathrm{MCMC}} = 160$ proposal forwards, and layer-subset probing $K/L$ of a forward per probed candidate.

\paragraph{The depth signal is nearly free.}
The probe is teacher-forced and touches only $K$ of $L$ layers, so it costs a fraction of a full forward pass rather than a whole one. On Qwen2.5-Math-7B / MATH500 it adds just $+3.3\%$ per-problem wall-clock to DEGS reranking and $+8.8\%$ to DEGS-MCMC over their likelihood-only counterparts (Table~\ref{tab:compute-accounting}). The accuracy gains of Table~\ref{tab:main-results} therefore come at single-digit-percent overhead---far below the near-doubling that a naive count of one extra forward per candidate would imply.

\paragraph{The signal is localizable in depth.}
Sweeping $\Lsub$ (full sweep in Table~\ref{tab:lsub-tradeoff}, Appendix~\ref{app:implementation}) separates two effects that ``probe fewer layers'' conflates---how \emph{many} layers enter the grid versus \emph{where} they sit. Count is almost free: thinning the default seven-layer grid down to five leaves MATH500 and GPQA accuracy unchanged, and probing all $28$ layers buys only $+0.4$ point on MATH500 at $1.05\times$ cost. Placement, by contrast, is decisive: holding the count at four and sliding the window toward the input lowers accuracy monotonically---$76.6\!\to\!72.8$ on MATH500, with the same monotone decline on GPQA---until an input-only probe falls below even the $\beta{=}0$ baselines (Table~\ref{tab:core-ablation}). Because this ordering recurs across both out-of-domain GPQA and in-domain MATH500, it reflects \emph{where} entropy collapse carries discriminative content rather than a benchmark quirk, and the probe can safely be restricted to those layers at no accuracy penalty.

\subsection{Ablation Studies}

\begin{table*}[t]
\caption{Core ablation of DEGS components on Qwen2.5-Math-7B, isolating likelihood-based scoring, depth-based entropy scoring, and alternative entropy signals under a matched-compute protocol ($N=16$, MCMC steps $=10$, block size $B=16$).}
\label{tab:core-ablation}
\centering
\small
\setlength{\tabcolsep}{4.6pt}
\resizebox{\textwidth}{!}{
\begin{tabular}{lccccccc}
\toprule
Method & $\alpha \log p(\bx)$ & $\beta D(\bx)$ & Final-layer entropy & Selection mode & MATH500 $\uparrow$ & DeepMind500 $\uparrow$ & GPQA $\uparrow$ \\
\midrule
Best-of-$N$ (log-likelihood)      & \cmark & --     & --     & Rerank & 68.4 & 46.5 & 34.3 \\
Power Sampling ($\beta = 0$)      & \cmark & --     & --     & MCMC   & 74.8 & 57.4 & 38.9 \\
Depth-only scoring ($\alpha = 0$) & --     & \cmark & --     & Rerank & 63.2 & 41.4 & 32.8 \\
Best-of-$N$-Entropy (ours)        & --     & --     & \cmark & Rerank & 68.7 & 47.9 & 36.9 \\
DEGS reranking                    & \cmark & \cmark & --     & Rerank & 71.6 & 54.8 & 36.2 \\
DEGS-MCMC                         & \cmark & \cmark & --     & MCMC   & 76.7 & 61.9 & 42.6 \\
\bottomrule
\end{tabular}
}
\end{table*}

\paragraph{Both terms are necessary, and the signal lives in the depth profile, not the final layer.}
DEGS-MCMC improves over likelihood-only power sampling ($\beta{=}0$) on all three benchmarks---up to $+4.5$ on DeepMind500 (Table~\ref{tab:core-ablation})---and over depth-only scoring by a wide margin; neither term alone reproduces the combined result. The gain is not merely an entropy effect: final-layer-entropy reranking lands barely above Best-of-$N$ on MATH500 yet trails DEGS-MCMC by up to $14$ points on DeepMind500, so \emph{when} a sequence's predictions stabilize is more informative than its final-layer entropy alone (Appendix~\ref{app:design-choices})---the distinction that motivates collapse depth. Added to either selection mode, the depth term yields a positive marginal gain over its matched likelihood-only baseline on all three benchmarks, largest in reranking on the harder DeepMind500 split, as Proposition~\ref{prop:bounded-perturbation} predicts; but applying the tiebreak once (reranking) versus at each of $\sim\!160$ accept/reject steps (DEGS-MCMC) is what separates them---DEGS-MCMC dominates reranking on all three benchmarks, and reranking serves only as a controlled probe of \emph{where} the signal acts (Table~\ref{tab:lsub-tradeoff}).

\section{Discussion and Conclusion}
\label{sec:discussion}

We introduced DEGS, a training-free method that turns layer-wise entropy collapse depth into a test-time quality signal: stronger reasoners collapse later, which we operationalize as a pseudo-reward $D(\bx)$ tilting the power objective. DEGS does not merely favor uncertain sequences---the ablation (Table~\ref{tab:core-ablation}) confirms both terms are necessary.

\paragraph{Limitations and future work.}
Our evaluation centers on mathematical reasoning, with preliminary coding and scientific-QA results; generalization to open-ended or factual tasks remains open. DEGS also requires hidden-state access (open-weight models only), and the logit lens is an imperfect probe~\citep{wendler2024llamas}. Promising extensions include richer entropy-curve features via learned scorers, token-selective collapse depth, cross-architecture validation, and integration with efficient samplers such as Power-SMC~\citep{azizi2026powersmc}.

\section*{Reproducibility Statement}
All experiments use publicly available open-weight base models (Qwen2.5-7B, Qwen2.5-Math-7B, DeepSeek-Math-7B) and public benchmarks (MATH500, DeepMind500, HumanEval, GPQA-Diamond). The collapse-depth definition (Definitions~\ref{def:per-token-collapse}--\ref{def:sequence-collapse}), the DEGS objective~\eqref{eq:degs-target}, and both decoding variants are specified in full in Section~\ref{sec:method} and Algorithms~\ref{alg:degs-mcmc}--\ref{alg:degs-rerank}. Appendix~\ref{app:implementation} reports the complete harness---per-backbone layer counts and probe grids, the prompt/parser/grader for each benchmark, decoding knobs, the entropy-probing procedure, the GRPO training recipe, and the matched-budget protocol---together with default hyperparameters and search ranges (Table~\ref{tab:hyperparameters-full}) and the per-step MCMC protocol (Table~\ref{tab:mcmc-protocol}); every prompt, parser, and grader is shared verbatim across all decoding methods. Formal statements and complete proofs of all theoretical claims are in Appendix~\ref{app:proofs}.

\section*{Ethics Statement}
This work studies training-free decoding for reasoning on standard public benchmarks using open-weight models, and involves no new data collection or human subjects. It inherits the general risks of LLM reasoning systems (e.g., confidently stated incorrect answers) but, to our knowledge, introduces no additional ethical concerns; because DEGS requires hidden-state access, it applies only to open-weight models.

\clearpage
\appendix

\section{Additional Figures and Tables}
\label{app:additional}

\begin{figure*}[t]
\centering
\setlength{\tabcolsep}{4pt}
\renewcommand{\arraystretch}{1.05}
\begin{tabular}{cccc}
\begin{minipage}[b]{0.235\textwidth}\centering
  \incfig{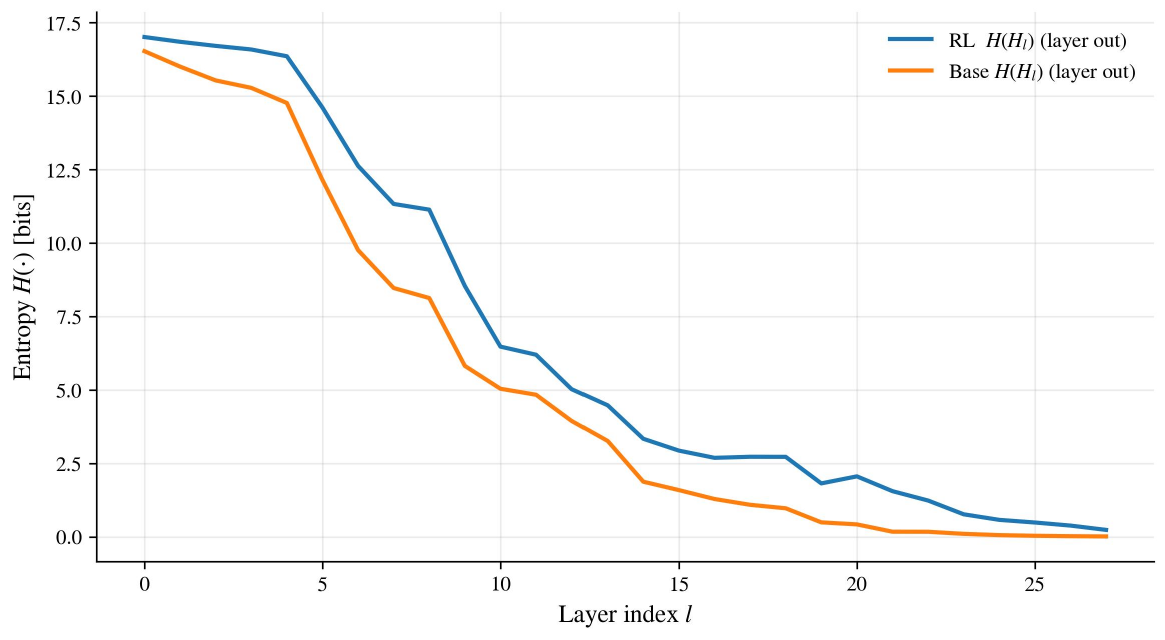}{0.95\linewidth}{1.08in}\\[2pt]
  {\small\textbf{(a)} Qwen2.5-Math-7B\\ MATH500}
\end{minipage}&
\begin{minipage}[b]{0.235\textwidth}\centering
  \incfig{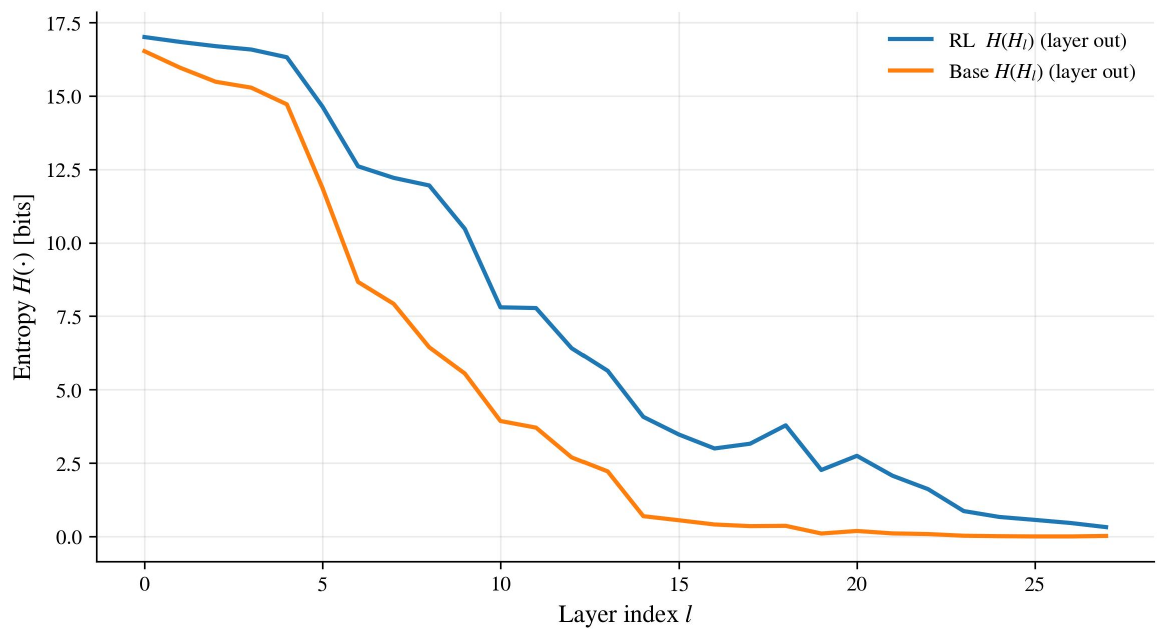}{0.95\linewidth}{1.08in}\\[2pt]
  {\small\textbf{(b)} Qwen2.5-Math-7B\\ DeepMind500}
\end{minipage}&
\begin{minipage}[b]{0.235\textwidth}\centering
  \incfig{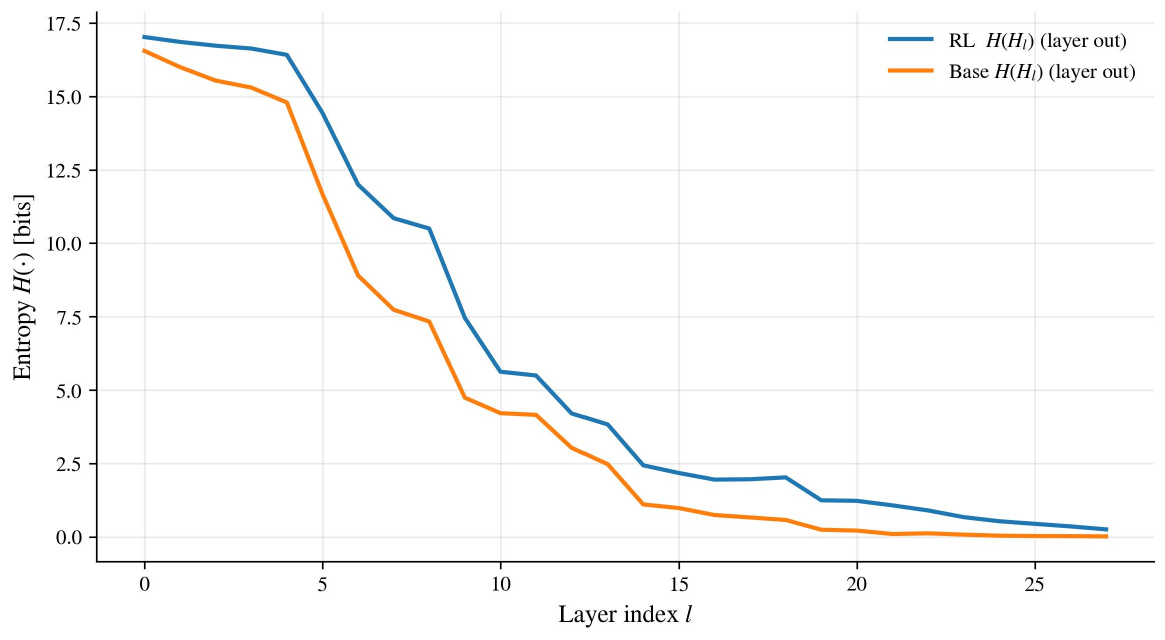}{0.95\linewidth}{1.08in}\\[2pt]
  {\small\textbf{(c)} Qwen2.5-Math-7B\\ GSM8K}
\end{minipage}&
\begin{minipage}[b]{0.235\textwidth}\centering
  \incfig{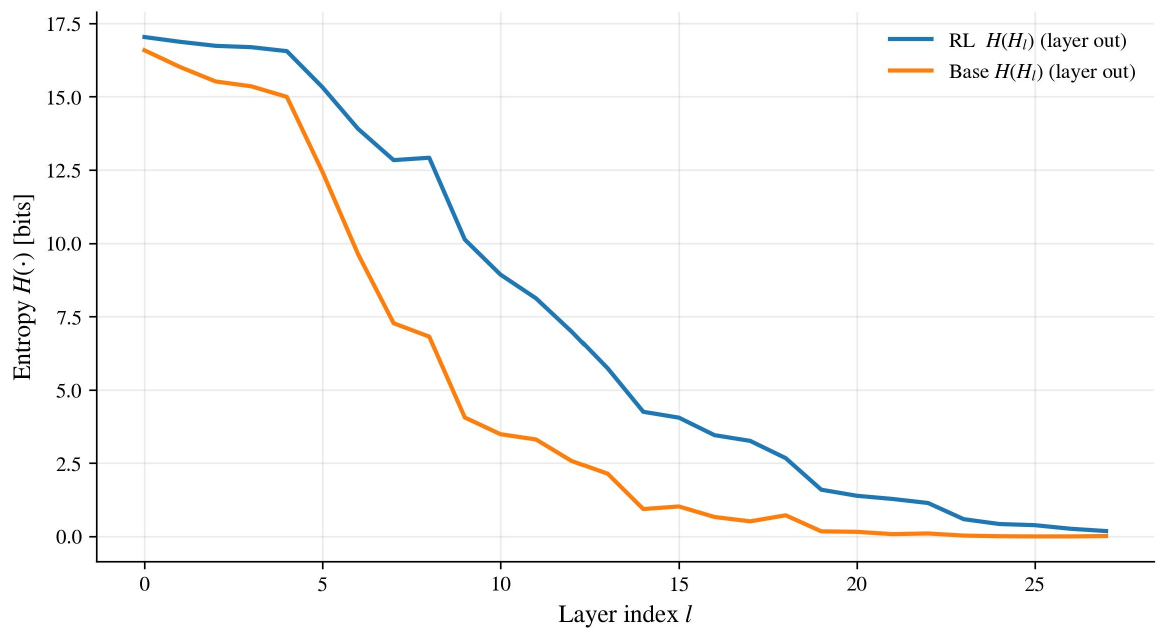}{0.95\linewidth}{1.08in}\\[2pt]
  {\small\textbf{(d)} Qwen2.5-Math-7B\\ GPQA}
\end{minipage}\\[6pt]
\begin{minipage}[b]{0.235\textwidth}\centering
  \incfig{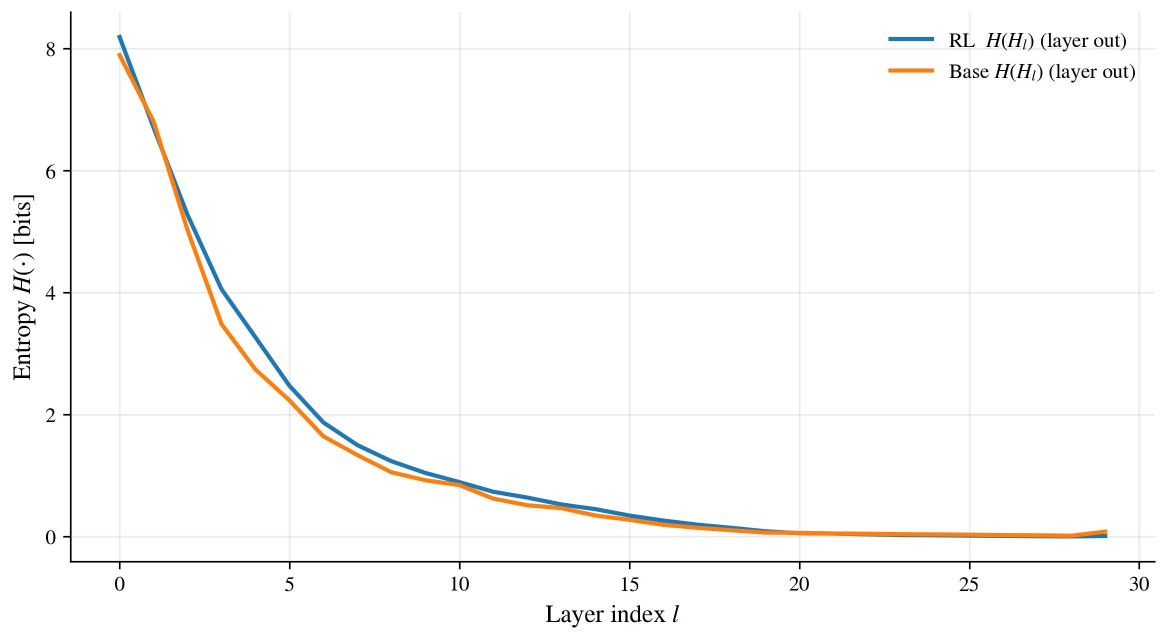}{0.95\linewidth}{1.08in}\\[2pt]
  {\small\textbf{(e)} DeepSeek-Math-7B\\ MATH500}
\end{minipage}&
\begin{minipage}[b]{0.235\textwidth}\centering
  \incfig{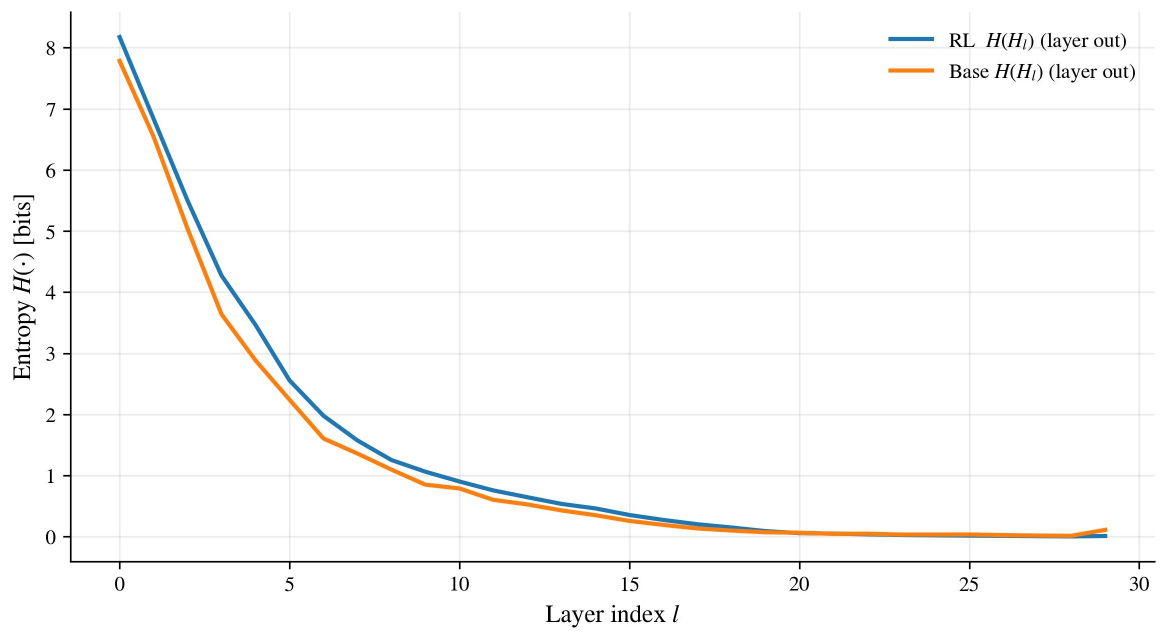}{0.95\linewidth}{1.08in}\\[2pt]
  {\small\textbf{(f)} DeepSeek-Math-7B\\ DeepMind500}
\end{minipage}&
\begin{minipage}[b]{0.235\textwidth}\centering
  \incfig{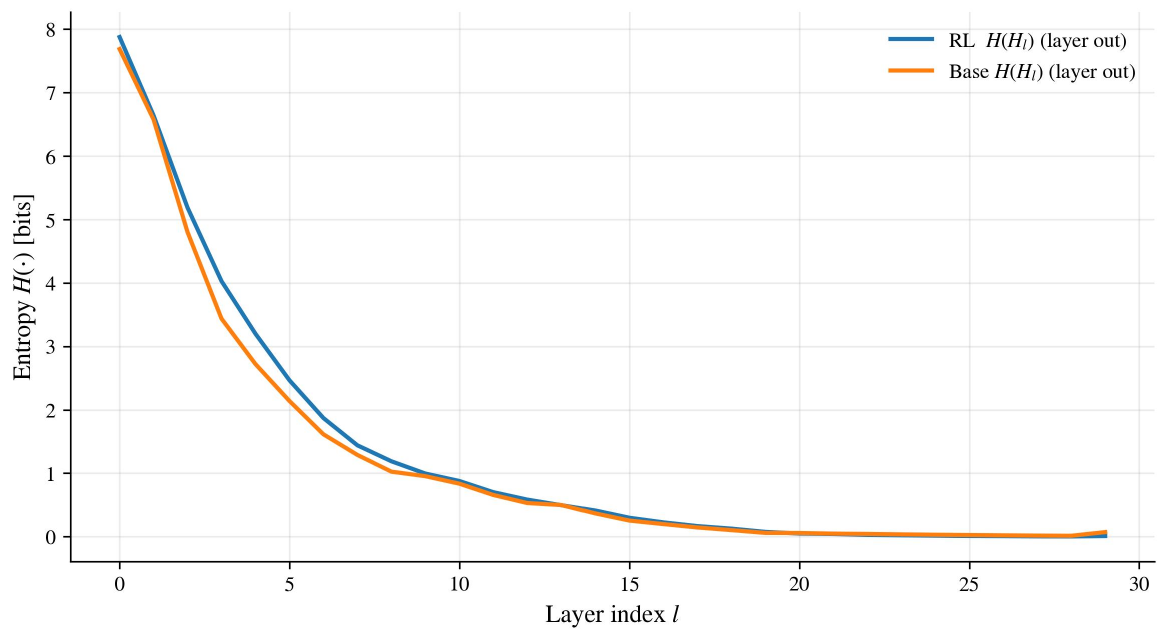}{0.95\linewidth}{1.08in}\\[2pt]
  {\small\textbf{(g)} DeepSeek-Math-7B\\ GSM8K}
\end{minipage}&
\begin{minipage}[b]{0.235\textwidth}\centering
  \incfig{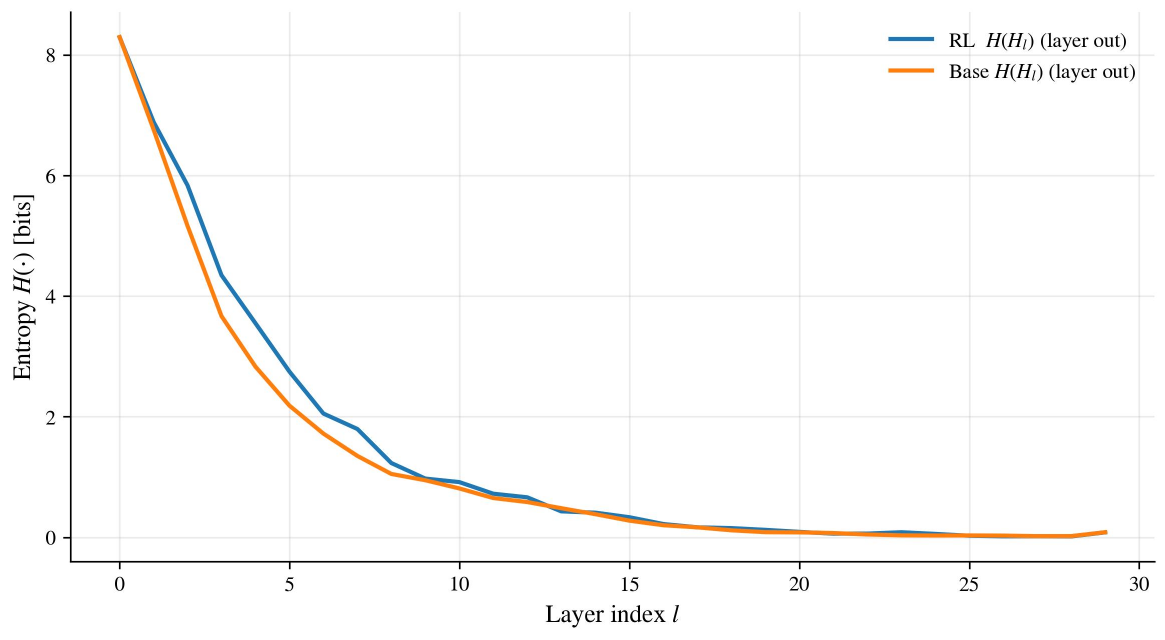}{0.95\linewidth}{1.08in}\\[2pt]
  {\small\textbf{(h)} DeepSeek-Math-7B\\ GPQA}
\end{minipage}\\[6pt]
\begin{minipage}[b]{0.235\textwidth}\centering
  \incfig{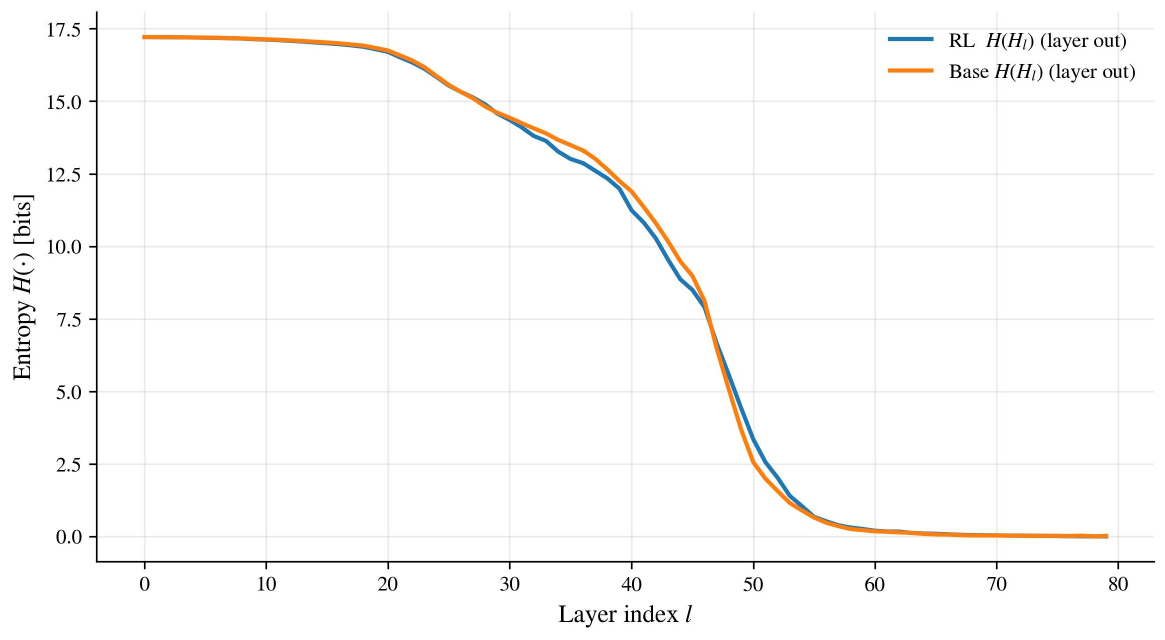}{0.95\linewidth}{1.08in}\\[2pt]
  {\small\textbf{(i)} Qwen2.5-Math-72B\\ MATH500}
\end{minipage}&
\begin{minipage}[b]{0.235\textwidth}\centering
  \incfig{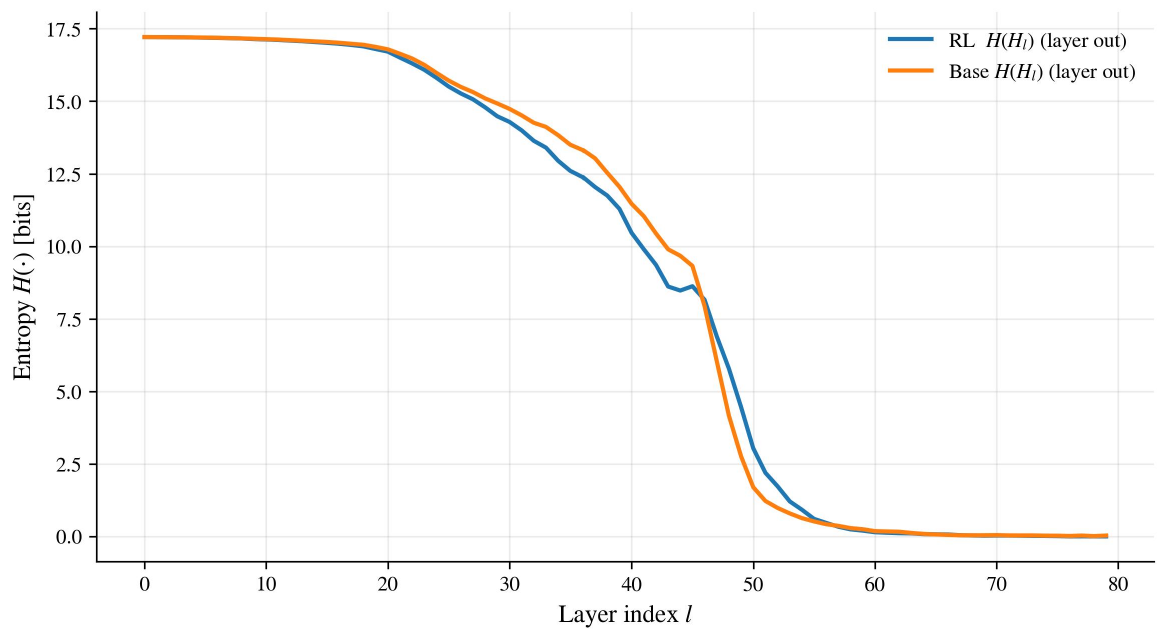}{0.95\linewidth}{1.08in}\\[2pt]
  {\small\textbf{(j)} Qwen2.5-Math-72B\\ DeepMind500}
\end{minipage}&
\begin{minipage}[b]{0.235\textwidth}\centering
  \incfig{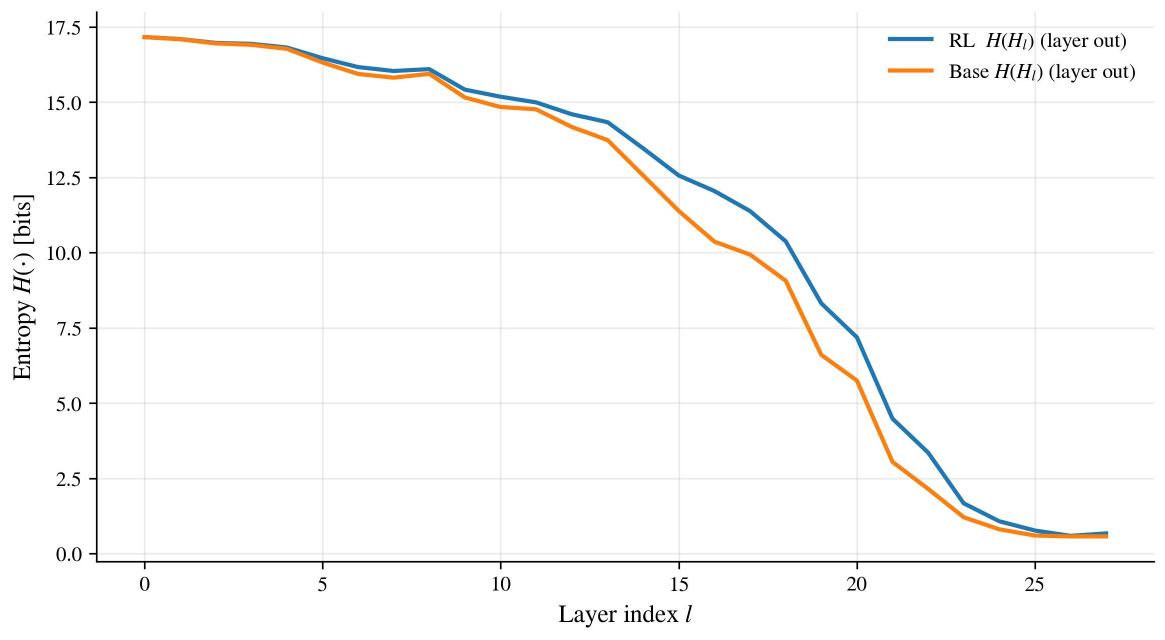}{0.95\linewidth}{1.08in}\\[2pt]
  {\small\textbf{(k)} Qwen2.5-7B\\ MATH500}
\end{minipage}&
\begin{minipage}[b]{0.235\textwidth}\centering
  \incfig{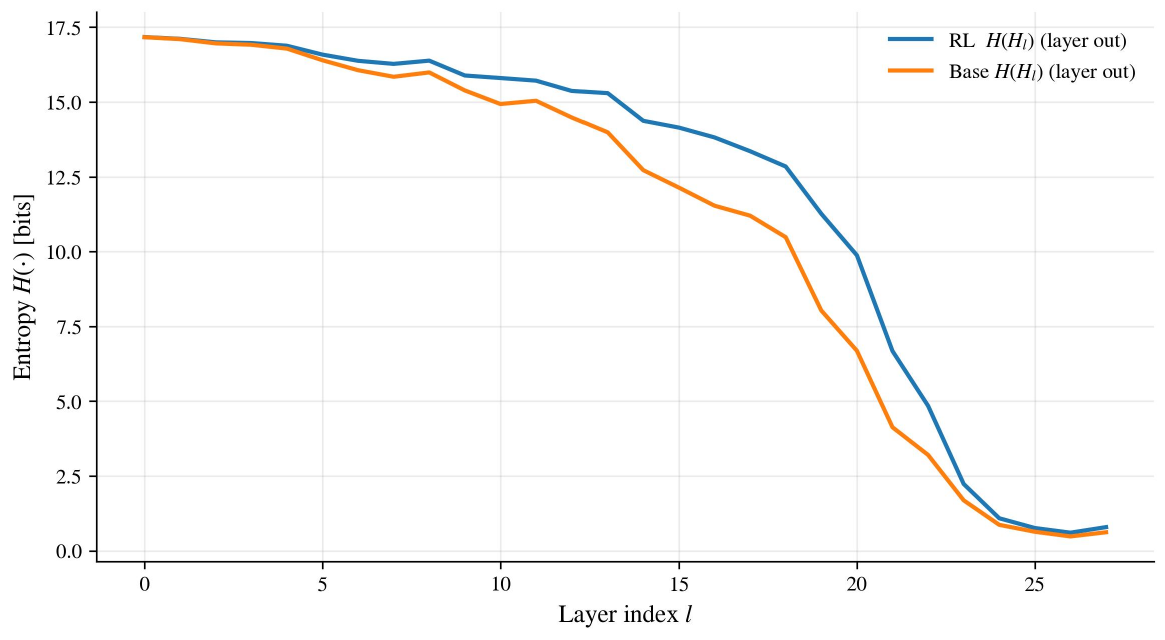}{0.95\linewidth}{1.08in}\\[2pt]
  {\small\textbf{(l)} Qwen2.5-7B\\ DeepMind500}
\end{minipage}
\end{tabular}
\caption{Complete entropy-versus-layer visualizations across all twelve (model, benchmark) pairs. Each panel overlays the layer-wise token entropy $H_l$ in nats for the posttrained variant against its base counterpart, measured on the subset where the posttrained model is correct and the base model is wrong; this isolates the depths at which posttraining adds discriminative content. This appendix figure complements Figure~\ref{fig:late-collapse} by expanding coverage and reducing the possibility that the main-paper examples are selective.}
\label{fig:appendix-entropy-grid}
\end{figure*}

\begin{table*}[t]
\caption{Full benchmark results with all reported settings. Same data as
Table~\ref{tab:main-results}, with the decoding budget reported as
$N{=}\text{candidates}$, $M{=}\text{MCMC steps}$, $T{=}\text{temperature}$.
Power-SMC is included as an additional training-free decoding
baseline~\citep{azizi2026powersmc}.}
\label{tab:appendix-full-results}
\centering
\footnotesize
\setlength{\tabcolsep}{4.0pt}
\resizebox{\textwidth}{!}{
\begin{tabular}{llccccc}
\toprule
Model & Method & Budget & MATH500 $\uparrow$ & DeepMind500 $\uparrow$ & HumanEval Pass@1 $\uparrow$ & GPQA $\uparrow$ \\
\midrule
Qwen2.5-7B & Base                          & $T{=}0$                  & 0.498 & 0.250 & 0.329 & 0.278 \\
Qwen2.5-7B & Low-temperature               & $T{=}0.25$               & 0.628 & 0.342 & 0.524 & 0.303 \\
Qwen2.5-7B & Best-of-$N$                   & $N{=}16$                 & 0.650 & 0.395 & 0.609 & 0.282 \\
Qwen2.5-7B & MCMC Power Sampling           & $N{=}16,M{=}10$          & 0.706 & 0.537 & 0.622 & 0.318 \\
Qwen2.5-7B & Scalable Power Sampling       & ($T{=}0.25$)             & 0.708 & 0.549 & 0.756 & 0.349 \\
Qwen2.5-7B & Power-SMC                     & $N{=}16$                 & 0.714 & 0.563 & 0.761 & 0.351 \\
Qwen2.5-7B & Best-of-$N$-Entropy           & $N{=}16$                 & 0.664 & 0.412 & 0.632 & 0.348 \\
Qwen2.5-7B & DEGS-MCMC                     & $N{=}16,M{=}10,\beta{=}5$ & 0.728 & 0.573 & 0.769 & 0.393 \\
Qwen2.5-7B & GRPO (MATH)                   & $T{=}0$                  & 0.740 & 0.598 & 0.561 & 0.354 \\
\midrule
Qwen2.5-Math-7B & Base                          & $T{=}0$                  & 0.496 & 0.332 & 0.329 & 0.278 \\
Qwen2.5-Math-7B & Low-temperature               & $T{=}0.25$               & 0.690 & 0.482 & 0.512 & 0.353 \\
Qwen2.5-Math-7B & Best-of-$N$                   & $N{=}16$                 & 0.684 & 0.465 & 0.512 & 0.343 \\
Qwen2.5-Math-7B & MCMC Power Sampling           & $N{=}16,M{=}10$          & 0.748 & 0.574 & 0.573 & 0.389 \\
Qwen2.5-Math-7B & Scalable Power Sampling       & ($T{=}0.25$)             & 0.758 & 0.593 & 0.604 & 0.409 \\
Qwen2.5-Math-7B & Power-SMC                     & $N{=}16$                 & 0.762 & 0.601 & 0.612 & 0.413 \\
Qwen2.5-Math-7B & Best-of-$N$-Entropy           & $N{=}16$                 & 0.687 & 0.479 & 0.537 & 0.369 \\
Qwen2.5-Math-7B & DEGS-MCMC                     & $N{=}16,M{=}10,\beta{=}5$ & 0.767 & 0.619 & 0.618 & 0.426 \\
Qwen2.5-Math-7B & GRPO (MATH)                   & $T{=}0$                  & 0.785 & 0.632 & 0.537 & 0.399 \\
\midrule
DeepSeek-Math-7B & Base                          & $T{=}0$                  & 0.362 & 0.212 & 0.415 & 0.333 \\
DeepSeek-Math-7B & Low-temperature               & $T{=}0.25$               & 0.366 & 0.269 & 0.427 & 0.430 \\
DeepSeek-Math-7B & Best-of-$N$                   & $N{=}16$                 & 0.420 & 0.276 & 0.433 & 0.338 \\
DeepSeek-Math-7B & MCMC Power Sampling           & $N{=}16,M{=}10$          & 0.424 & 0.307 & 0.470 & 0.345 \\
DeepSeek-Math-7B & Scalable Power Sampling       & ($T{=}0.25$)             & 0.464 & 0.314 & 0.487 & 0.364 \\
DeepSeek-Math-7B & Power-SMC                     & $N{=}16$                 & 0.467 & 0.329 & 0.498 & 0.386 \\
DeepSeek-Math-7B & Best-of-$N$-Entropy           & $N{=}16$                 & 0.423 & 0.298 & 0.453 & 0.341 \\
DeepSeek-Math-7B & DEGS-MCMC                     & $N{=}16,M{=}10,\beta{=}5$ & 0.475 & 0.336 & 0.508 & 0.392 \\
DeepSeek-Math-7B & GRPO (MATH)                   & $T{=}0$                  & 0.492 & 0.352 & 0.524 & 0.333 \\
\bottomrule
\end{tabular}
}
\end{table*}

\begin{table*}[t]
\caption{DEGS-MCMC protocol details: the practical setup of the entropy-guided MCMC variant, including proposal construction, chain length, acceptance computation, and stopping rules.}
\label{tab:mcmc-protocol}
\centering\small\setlength{\tabcolsep}{4.6pt}
\resizebox{\textwidth}{!}{
\begin{tabular}{l p{10.3cm}}
\toprule
Component & Specification \\
\midrule
Target distribution & $\pi(\bx) \propto p(\bx)^{\alpha}\exp(\beta D(\bx))$, with $\alpha{=}1/\tau_{\mathrm{prop}}$ (default $\alpha{=}4$, $\tau_{\mathrm{prop}}{=}0.25$) and $\beta{\ge}0$ (default $\beta{=}5$; $\beta{=}0$ recovers plain MCMC power sampling) \\
Proposal family & Local suffix resampling: at each MH step, draw a cut index $i$ and regenerate all tokens from $i$ to the end of the current completion with the base model at temperature $\tau_{\mathrm{prop}}$ (multinomial, no top-$k$/top-$p$) \\
Proposal granularity & Token-level from a random cut point; the resampled span length is implicitly variable and is corrected for in the MH ratio \\
Editable span & Only the generated completion $[|\mathrm{prompt}|,|\mathrm{gen}|)$ is editable; the prompt is frozen, so the cut index is drawn uniformly from $\{|\mathrm{prompt}|,\dots,|\mathrm{gen}|{-}1\}$ \\
Initialization & Block-wise warm start: the chain is built incrementally over $B{=}16$ blocks of $192$ tokens each, with each block sampled at temperature $\tau_{\mathrm{prop}}$ before MH steps run \\
Chain length $T_{\mathrm{MCMC}}$ & $T_{\mathrm{MCMC}}{=}10$ MH steps per block, $B{=}16$ blocks; total $160$ proposals per sample at default \\
Burn-in / warm start & No explicit burn-in: block-wise warm start already places the chain in a high-likelihood region \\
Acceptance ratio & $\log r = \alpha\bigl(\log p(\bx')-\log p(\bx)\bigr)+\bigl(\log q(\bx\mid \bx')-\log q(\bx'\mid \bx)\bigr)+\beta\bigl(D(\bx')-D(\bx)\bigr)$; accept with probability $\min(1,e^{\log r})$. The DEGS correction $\beta(D(\bx')-D(\bx))$ is additive and active only when $\beta{>}0$. Measured acceptance at the canonical setting ($\alpha{=}4$, $\beta{=}5$): deepseek\_math/deepmind: 78\%; qwen\_math/math: 57\%. \\
Likelihood evaluation & Single teacher-forced forward per proposal; base logprobs are summed over the completion without length normalization \\
Entropy probing for $D(\bx)$ & Separate no-gradient forward with hidden-state hooks on $\Lsub$; at each layer the logit lens is applied after final RMSNorm, and exact token entropy is computed in vocab chunks of $8192$; $D(\bx)$ is the mean normalized depth over generated (non-EOS) tokens \\
Caching / reuse strategy & Current state's $D(\bx)$ is cached and refreshed only when a new block is appended or a proposal is accepted; each proposal recomputes $D(\bx')$ from scratch \\
Early stopping & Chain terminates when EOS appears; otherwise the full $B{\cdot}T_{\mathrm{MCMC}}$ proposals are run \\
Selection from chain & Final chain state is returned (no chain averaging or best-so-far selection) \\
Budget normalization & Reported in forward-equivalent units as defined in Table~\ref{tab:compute-accounting} \\
\bottomrule
\end{tabular}
}
\end{table*}

\section{Preliminaries (Full)}
\label{sec:preliminaries-full}

\paragraph{Autoregressive language models.}
Let $\mathcal{X}$ denote a finite vocabulary. An autoregressive LLM defines $p(\bx) = \prod_{t=0}^{T} p(x_t \mid x_{<t})$. In a transformer with $L$ layers, hidden states $\mathbf{h}_{l,t} \in \mathbb{R}^d$ are updated via $\mathbf{h}_{l,t} = \mathbf{h}_{l-1,t} + f_l(\mathbf{h}_{l-1,1}, \ldots, \mathbf{h}_{l-1,t})$, where $f_l$ is the $l$-th transformer block. After layer $L$, the unembedding matrix $U \in \mathbb{R}^{|\mathcal{X}| \times d}$ produces next-token logits.

\paragraph{Logit lens.}
The logit lens~\citep{nostalgebraist2020logitlens} applies $U$ at intermediate layers: $p_l(v \mid x_{\le t}) = \mathrm{softmax}(U\mathbf{h}_{l,t})_v$. The layer-wise entropy is $H_l(t) = -\sum_v p_l(v \mid x_{\le t}) \log p_l(v \mid x_{\le t})$. Prior work shows entropy is high in early layers and drops sharply in later layers~\citep{wendler2024llamas}.

\paragraph{Power distributions.}
The power distribution $p^\alpha(\bx) \propto p(\bx)^\alpha$ ($\alpha \ge 1$) sharpens the base distribution. \citet{karan2025reasoning} show that $p^\alpha$ differs from low-temperature sampling and propose MH sampling to target it.

\section{DEGS Reranking Algorithm}
\label{app:rerank-algorithm}

Algorithm~\ref{alg:degs-rerank} gives the depth-based reranking procedure referenced in Section~\ref{sec:degs-reranking}. Given a problem prompt, sample $N$ candidate solutions $\{\bx^{(i)}\}_{i=1}^N$ from the base model $p$ using moderate randomness (temperature $\in[0.7,1.0]$, top-$p=0.95$), recording the per-token log-probabilities. For each candidate, compute the log-likelihood from the recorded log-probs and the collapse depth $D(\bx^{(i)})$ from one additional teacher-forced forward pass (logit-lens entropy at layers $\Lsub$), then the combined score $S^{(i)}=\alpha\log p(\bx^{(i)})+\beta D(\bx^{(i)})$, and return $\bx^*=\argmax_i S^{(i)}$. As discussed in Section~\ref{sec:degs-reranking}, at $N{=}16$ this is not a competitive standalone decoder (Proposition~\ref{prop:bounded-perturbation}); we retain it as a controlled probe of \emph{where} in depth the collapse signal carries discriminative content (Table~\ref{tab:lsub-tradeoff}).

\begin{algorithm}[h]
\caption{DEGS Reranking (Best-of-$N$)}
\label{alg:degs-rerank}
\begin{algorithmic}[1]
\REQUIRE Base model $p$; prompt; candidate budget $N$; layer subset $\Lsub$; threshold $\tau$; weights $\alpha, \beta$
\ENSURE Selected solution $\bx^*$
\STATE Sample $N$ candidates $\{\bx^{(i)}\}_{i=1}^N \sim p(\cdot \mid \text{prompt})$; record log-probs
\FOR{$i = 1, \ldots, N$}
  \STATE $\log p(\bx^{(i)}) \gets \sum_{t} \log p(x_t^{(i)} \mid x_{<t}^{(i)})$ \hfill \COMMENT{from sampling}
  \STATE Run forward pass on prompt $\oplus\; \bx^{(i)}$ with hidden states at layers $\Lsub$
  \FOR{each layer $l \in \Lsub$ and position $t$}
    \STATE $H_l(t) \gets -\sum_{v} p_l(v \mid x_{\le t}^{(i)}) \log p_l(v \mid x_{\le t}^{(i)})$ \hfill \COMMENT{logit-lens entropy}
  \ENDFOR
  \FOR{each position $t$}
    \STATE $d_t \gets \min\{l \in \Lsub : H_l(t) \le \tau\}$ \quad (default $L$ if none)
  \ENDFOR
  \STATE $D(\bx^{(i)}) \gets \frac{1}{T^{(i)}+1}\sum_{t} d_t / L$
  \STATE $S^{(i)} \gets \alpha \cdot \log p(\bx^{(i)}) + \beta \cdot D(\bx^{(i)})$
\ENDFOR
\RETURN $\bx^* \gets \argmax_i S^{(i)}$
\end{algorithmic}
\end{algorithm}

\section{Design Choices and Practical Considerations}
\label{app:design-choices}

\paragraph{Relationship to existing signals.}
DEGS differs from final-layer entropy reranking in that it captures the \emph{depth profile} of entropy evolution, not merely its endpoint. A candidate may have low final-layer entropy yet collapse early, reflecting superficial rather than deep reasoning. Table~\ref{tab:core-ablation} isolates the contribution of each component.

\section{Implementation Details}
\label{app:implementation}

This appendix records the decoding, probing, training, and evaluation
configuration behind Section~\ref{sec:experiments}. The tables in this section
are referenced from the main text; the prose below specifies what the tables do
not, so that every reported number is reproducible from a single fixed harness.

\paragraph{Base models.}
All experiments use three open-weight base checkpoints: Qwen2.5-7B and
Qwen2.5-Math-7B, each with $L{=}28$ transformer layers, and DeepSeek-Math-7B
with $L{=}30$ layers. Because collapse depth is read off a fixed grid of layer
indices, the default probe $\Lsub$ is rescaled to each backbone's depth---
$\{4,8,12,16,20,24,28\}$ for the two $28$-layer Qwen models and
$\{4,9,14,18,22,26,30\}$ for the $30$-layer DeepSeek model---so that the seven
probed layers occupy comparable \emph{relative} depths $l/L$ across
architectures rather than a fixed absolute index.

\paragraph{Benchmarks, prompts, and grading.}
Table~\ref{tab:datasets-eval-full} gives the output format and answer-processing pipeline for each benchmark; the governing principle is that the harness is held fixed while only the decoder varies. MATH500 and DeepMind500 share an identical zero-shot chain-of-thought prompt, parser, and grader, with the final answer read from the last \texttt{\textbackslash boxed\{\}} span and scored by a sympy-based symbolic/numeric matcher; the two splits therefore differ only in problem difficulty, which is what licenses reading the larger DEGS gains on DeepMind500 (Section~\ref{sec:main-results}) as a difficulty effect rather than a harness artifact. GPQA (Diamond) is posed as a four-way multiple-choice task whose option order is shuffled once and held fixed across methods---so every decoder is scored against the same gold position---with the predicted letter read from the boxed answer, matched exactly against the gold letter, and completions capped at $3072$ tokens. HumanEval is generated by direct completion from the function signature, without chain-of-thought, and graded for functional correctness (Pass@1) by the official test harness. Every prompt template, parser, and grader is shared verbatim across all decoding methods and all models, so that accuracy differences between rows of Table~\ref{tab:main-results} isolate the decoder.

\paragraph{Decoding configuration.}
Each method samples from the same base model under a shared set of decoding
knobs. We disable nucleus~\citep{holtzman2020nucleus} and top-$k$ truncation throughout (top-$p{=}1.0$): the
power-sampling family targets the untruncated base distribution, and matching
this across baselines keeps the comparison on equal footing. The likelihood
exponent is realized as an inverse sampling temperature ($\alpha{=}1/T$), so the
default $\alpha{=}4$ corresponds to $T{=}0.25$ on MATH500, DeepMind500, and
GPQA, whereas HumanEval is decoded at $T{=}0.5$ ($\alpha{=}2$) to retain the
diversity program synthesis requires. The reranking methods (Best-of-$N$,
Best-of-$N$-Entropy, and DEGS reranking) draw $N{=}16$ candidates and
return one; the MCMC family (power sampling and DEGS-MCMC) runs
$T_{\mathrm{MCMC}}{=}10$ Metropolis--Hastings steps in each of $B{=}16$ blocks of
$192$ tokens, i.e.\ $160$ proposals per sample. All DEGS variants use depth
weight $\beta{=}5$ and their likelihood-only counterparts $\beta{=}0$, so that
any DEGS-versus-baseline difference is attributable solely to the depth-entropy
term. Defaults and search ranges are collected in
Table~\ref{tab:hyperparameters-full}.

\paragraph{Collapse-depth probing.}
Scoring a completion with $\Dx$ adds one teacher-forced forward pass that caches
hidden states at the layers in $\Lsub$. At each probed layer the model's final
RMSNorm is applied before the logit lens---matching the normalization used at
the output head---and exact token entropy is computed over the full vocabulary
in chunks of $8192$ entries to bound peak memory. The per-token collapse depth
$d_t$ is the shallowest probed layer whose logit-lens entropy is at or below
$\tau{=}0.25$ nats ($d_t{=}L$ if none qualifies), and $\Dx$ averages $d_t/L$ over
all generated tokens, excluding EOS and scoring the full completion span. In
DEGS-MCMC the current state's $\Dx$ is cached and refreshed only when a block is
appended or a proposal is accepted, whereas every proposal recomputes
$D(\bx')$ from scratch; the complete per-step protocol, including the acceptance
ratio and stopping rule, is given in Table~\ref{tab:mcmc-protocol}.

\paragraph{Predictive-validity sampling.}
The analysis of whether $\Dx$ predicts correctness
(Figure~\ref{fig:collapse-depth-predicts-correctness}) deliberately uses a
higher-entropy regime than the decoding runs---$N{=}16$ candidates per problem at
$T{=}0.8$, top-$p{=}0.9$---so that the candidate pool spans a wide range of $\Dx$
and the correlation is not estimated on a near-degenerate set of near-identical
samples.

\paragraph{GRPO reference.}
Following \citet{karan2025reasoning}, our GRPO references are produced rather than quoted: we posttrain each base model with GRPO~\citep{shao2024deepseekmath} and evaluate the resulting policy through the same harness as every other method. The training setup mirrors theirs---we adopt the GRPO implementation of \citet{shao2025spurious} with its default hyperparameters, optimize against a ground-truth math verifier on the MATH training split, and use a group size of $16$ rollouts per prompt. For the two Qwen backbones this configuration coincides with \citet{karan2025reasoning}; DeepSeek-Math-7B is outside their study, so its GRPO reference is trained by us under the same recipe. Decoding is greedy throughout, so that GRPO functions as an in-domain reference obtained under matched conditions.

\paragraph{Sharding and compute.}
Generation is distributed over a heterogeneous GPU cluster of RTX~5880 and
RTX~3090 nodes, one worker per GPU; the shards---five for MATH500 and
DeepMind500, six for GPQA, and four for HumanEval---are a throughput split over
the problem set only. Hardware is split by experiment: the main accuracy results
are measured on the RTX~5880 nodes, whereas the compute-efficiency sweep behind
Table~\ref{tab:lsub-tradeoff} and Table~\ref{tab:compute-accounting} runs on a
single RTX~3090---a difference that moves wall-clock alone and not accuracy,
since the decoder and harness are identical across nodes. Test-time compute is
normalized to forward-equivalent units as defined in
Table~\ref{tab:compute-accounting}, under which each reranking and MCMC method is
matched against its same-family baseline; the wall-clock figures in that table
are per-problem means on Qwen2.5-Math-7B / MATH500, and the corresponding
per-problem mean on GPQA under the MCMC setting is approximately $0.63$\,ks.

\begin{table*}[t]
\caption{Benchmarks, metrics, and evaluation setup, with output format and answer-processing details for each benchmark.}
\label{tab:datasets-eval-full}
\centering\small\setlength{\tabcolsep}{4.2pt}
\resizebox{\textwidth}{!}{
\begin{tabular}{lccccp{5.2cm}}
\toprule
Benchmark & Task type & Size & Primary metric & Output format & Evaluation / answer-processing setup \\
\midrule
MATH500 & Math reasoning & 500 & Accuracy & CoT + \texttt{\textbackslash boxed\{\}} & Zero-shot CoT prompt; answer parsed from last \texttt{\textbackslash boxed\{\}} span; graded by sympy-based symbolic/numeric matcher; temperature $0.25$ \\
DeepMind500 & Math reasoning & 500 & Accuracy & CoT + \texttt{\textbackslash boxed\{\}} & Held-out subset we construct from the DeepMind mathematics dataset; identical prompt, parser, and grader as MATH500; temperature $0.25$ \\
HumanEval & Code generation & 164 & Pass@1 & Python function & Direct completion from function signature (no CoT); code extracted via regex; graded by official functional-correctness harness (per-test timeout $3$\,s); temperature $0.5$ \\
GPQA (Diamond) & Scientific QA & 198 & Accuracy & MC (A/B/C/D) + CoT & Zero-shot CoT MC prompt; letter extracted from final \texttt{\textbackslash boxed\{\}} span; temperature $0.25$ \\
\bottomrule
\end{tabular}
}
\end{table*}

\begin{table*}[t]
\caption{Compute accounting and matched-budget protocol: how test-time compute is normalized across methods. All MCMC-family methods use $B{=}16$ blocks with $T_{\mathrm{MCMC}}{=}10$ MH steps per block ($160$ total proposals), and all layer-subset probing is charged as $K/L$ of a full forward pass per probed candidate (default $K{=}7$, $L{\in}\{28,30\}$). The wall-clock column reports per-problem mean elapsed seconds on Qwen2.5-Math-7B / MATH500 measured during the compute-efficiency sweep reported in Table~\ref{tab:main-results}.}
\label{tab:compute-accounting}
\centering\small\setlength{\tabcolsep}{4.0pt}
\resizebox{\textwidth}{!}{
\begin{tabular}{lccccccc}
\toprule
Method & Candidate budget & Extra forward passes & Hidden-state probing & MCMC steps & Budget unit & Matched-budget note & Wall-clock \\
\midrule
Base decoding & $1$ sample & $0$ & No & $0$ & $1{\times}$ decode pass & Single-sample reference & -- \\
Low-temperature decoding & $N{=}16$ samples & $0$ & No & $0$ & $N{\times}$ decode passes & Matched by candidate count & -- \\
Best-of-$N$ (log-likelihood) & $N{=}16$ samples & $0$ & No & $0$ & $N$ decode passes & Same candidate set as DEGS rerank & 228.1\,s \\
Best-of-$N$-Entropy & $N{=}16$ samples & $N$ & Yes (final layer) & $0$ & $N{+}N$ forward-equivalents & Same candidate pool and rerank budget & 230.5\,s \\
MCMC Power Sampling / SPS & $B{\cdot}T_{\mathrm{MCMC}}{=}160$ proposals & $0$ & No & $T_{\mathrm{MCMC}}{=}10$/block & Total proposal forwards & Matched by forward-equivalent cost & 330.5\,s \\
Power-SMC & Particle-dependent & $0$ & No & $T_{\mathrm{MCMC}}{=}10$/block & Total proposal forwards & Matched against MCMC family & -- \\
DEGS reranking & $N{=}16$ samples & $N$ ($\Lsub$) & Yes ($|\Lsub|{=}7$) & $0$ & $N{+}N{\cdot}(K/L)$ fwd-eq. & Matched against rerank baselines & 235.6\,s \\
DEGS-MCMC & $B{\cdot}T_{\mathrm{MCMC}}{=}160$ proposals & $1$ probe/MH step ($\Lsub$) & Yes ($|\Lsub|{=}7$) & $T_{\mathrm{MCMC}}{=}10$/block & Proposals $+\,B{\cdot}T_{\mathrm{MCMC}}{\cdot}(K/L)$ probes & Matched against MCMC family & 359.5\,s \\
\bottomrule
\end{tabular}
}
\end{table*}

\begin{table}[t]
\caption{Accuracy--overhead trade-off for the layer subset $\Lsub$: representative layer-subset choices for entropy probing with their measured wall-clock cost (per-problem mean, averaged over both (model, benchmark) evaluation settings used in Table~\ref{tab:compute-accounting}) and downstream accuracy. The DeepMind500 column is reported on DeepSeek-Math-7B; see Table~\ref{tab:compute-accounting} for the matched-budget protocol.}
\label{tab:lsub-tradeoff}
\centering
\small
\setlength{\tabcolsep}{4.0pt}
\begin{tabular}{lccccc}
\toprule
$\Lsub$ & $K$ & Relative overhead & MATH500 $\uparrow$ & GPQA $\uparrow$ & DeepMind500 $\uparrow$ \\
\midrule
All layers $\{1,\dots,L\}$ & 28 & 1.05$\times$ & 0.771 & 0.426 & 0.339 \\
$\{4,8,12,16,20,24,28\}$ & 7 & 1.00$\times$ & 0.767 & 0.426 & 0.336 \\
$\{8,12,16,20,24,28\}$ & 6 & 0.98$\times$ & 0.767 & 0.426 & 0.336 \\
$\{12,16,20,24,28\}$ & 5 & 0.98$\times$ & 0.767 & 0.426 & 0.336 \\
$\{16,20,24,28\}$ & 4 & 0.99$\times$ & 0.766 & 0.426 & 0.334 \\
$\{12,16,20,24\}$ & 4 & 0.98$\times$ & 0.765 & 0.424 & 0.334 \\
$\{8,12,16,20\}$ & 4 & 0.99$\times$ & 0.762 & 0.421 & 0.321 \\
$\{4,8,12,16\}$ & 4 & 0.99$\times$ & 0.759 & 0.418 & 0.320 \\
$\{1,4,8,12\}$ & 4 & 0.98$\times$ & 0.746 & 0.403 & 0.306 \\
$\{1,2,3,4\}$ & 4 & 0.98$\times$ & 0.728 & 0.385 & 0.275 \\
\bottomrule
\end{tabular}
\end{table}

\begin{table}[t]
\caption{Full DEGS hyperparameters and search ranges. The likelihood exponent $\alpha$ is implemented as $1/\mathrm{temperature}$ inside the power-sampling objective.}
\label{tab:hyperparameters-full}
\centering\small\setlength{\tabcolsep}{3.5pt}
\resizebox{\textwidth}{!}{
\begin{tabular}{lll}
\toprule
Parameter & Default & Search range \\
\midrule
Likelihood exponent $\alpha$ & $4$ & $\{2,\,4,\,6,\,8\}$ \\
Depth weight $\beta$ & $5$ & $\{0,\,1,\,2,\,5,\,10\}$ \\
Entropy threshold $\tau$ (nats) & $0.25$ & $\{0.05,\,0.10,\,\ldots,\,0.50\}$ \\
$\Lsub$ (Qwen, $L{=}28$) & $\{4,8,12,16,20,24,28\}$ & $K\!\in\!\{4,5,6,7\}$ + shifted \\
$\Lsub$ (DeepSeek, $L{=}30$) & $\{4,9,14,18,22,26,30\}$ & Depth-matched rescaling \\
Candidates $N$ & $16$ & $\{4,\,8,\,16,\,32\}$ \\
MCMC blocks $B$ & $16$ & fixed \\
MCMC steps/block $T_{\mathrm{MCMC}}$ & $10$ & $\{1,2,5,10,20,40\}$ \\
Block length & $192$ tokens & fixed \\
Sampling temperature & $0.25$; $0.5$ (HumanEval) & $\alpha{=}1/T$ \\
Top-$p$ & $1.0$ & fixed \\
Scoring span & Completion (excl.\ prompt) & \{full, compl., reasoning, answer, numeric\} \\
Apply final RMSNorm & Yes & fixed \\
Exclude EOS from entropy avg. & Yes & fixed \\
Vocab chunk size & $8192$ & fixed \\
\bottomrule
\end{tabular}
}
\end{table}

\section{Proofs and Derivations}
\label{app:proofs}

This appendix collects formal statements and proofs for the claims used in the main text. Throughout, $p$ is the base model, $\Dx\in(0,1]$ is the normalized collapse depth (Definition~\ref{def:sequence-collapse}), and $\pi(\bx)\propto p(\bx)^\alpha\exp(\beta\Dx)$ is the DEGS target (Definition~\ref{def:degs-target}) with $\alpha\ge 1$, $\beta\ge 0$.

\subsection{Normalizability of the DEGS target}

\begin{proposition}[Normalizability]\label{prop:normalizable}
Let the support $\mathcal{S}=\{\bx : p(\bx)>0\}$ be finite (e.g.\ sequences of bounded length over a finite vocabulary). Then $Z=\sum_{\bx\in\mathcal{S}} p(\bx)^\alpha \exp(\beta\Dx)$ satisfies $0<Z<\infty$, so $\pi(\bx)=p(\bx)^\alpha\exp(\beta\Dx)/Z$ is a well-defined probability distribution.
\end{proposition}

\begin{proof}
Each term is nonnegative, and positive for at least one $\bx$ (any $\bx\in\mathcal{S}$), so $Z>0$. Since $\Dx\le 1$, each term is bounded by $p(\bx)^\alpha e^{\beta}$; summing over the finite set $\mathcal{S}$ gives $Z\le e^{\beta}\sum_{\bx\in\mathcal{S}} p(\bx)^\alpha \le e^{\beta}\sum_{\bx\in\mathcal{S}} p(\bx) = e^{\beta}<\infty$, using $\alpha\ge 1$ and $p(\bx)\le 1$. Hence $0<Z<\infty$.
\end{proof}

\subsection{Convergence of DEGS-MCMC}
\label{app:convergence}

\begin{theorem}[Convergence to the DEGS target]\label{thm:convergence}
Consider the Metropolis--Hastings chain with target $\pi$ and the random-index resampling proposal $q$ of \citet{karan2025reasoning}, using the acceptance ratio~\eqref{eq:mh-degs}. If $q$ is irreducible and aperiodic on $\mathcal{S}$, then the chain has $\pi$ as its unique stationary distribution and the marginal law of the $n$-th state converges to $\pi$ in total variation as $n\to\infty$.
\end{theorem}

\begin{proof}
The acceptance ratio~\eqref{eq:mh-degs} is exactly the Metropolis--Hastings ratio for target $\pi$, because the unknown normalizer $Z$ cancels:
\[
\frac{\pi(\bx')\,q(\bx\mid\bx')}{\pi(\bx)\,q(\bx'\mid\bx)}
=\frac{p(\bx')^\alpha\exp(\beta D(\bx'))\,q(\bx\mid\bx')}{p(\bx)^\alpha\exp(\beta\Dx)\,q(\bx'\mid\bx)}.
\]
By Proposition~\ref{prop:normalizable}, $\pi$ is a well-defined distribution, and standard Metropolis--Hastings theory~\citep{neal1993probabilistic} guarantees that the chain satisfies detailed balance with respect to $\pi$. Detailed balance makes $\pi$ stationary; irreducibility and aperiodicity of the proposal (inherited by the accepted chain, since the acceptance probability is positive wherever $\pi$ is) make the stationary distribution unique and yield convergence in total variation by the ergodic theorem for Markov chains.
\end{proof}

\subsection{Power vs.\ temperature sampling}

\begin{proposition}[Power sampling is not temperature sampling]\label{prop:power-vs-temp}
Let $p(\bx)=\prod_{t=1}^{T}p(x_t\mid x_{<t})$ be autoregressive. The sequence-level power distribution $p^\alpha(\bx)\propto p(\bx)^\alpha$ is in general \emph{not} equal to the autoregressive distribution obtained by temperature sampling at temperature $1/\alpha$, i.e.\ the model $r(\bx)=\prod_{t}r(x_t\mid x_{<t})$ with $r(x_t\mid x_{<t})\propto p(x_t\mid x_{<t})^{\alpha}$, except in degenerate cases.
\end{proposition}

\begin{proof}
Temperature sampling renormalizes \emph{per step}: $r(x_t\mid x_{<t})=p(x_t\mid x_{<t})^\alpha / Z_t(x_{<t})$ with $Z_t(x_{<t})=\sum_{v}p(v\mid x_{<t})^\alpha$. Hence
\[
r(\bx)=\prod_{t=1}^{T}\frac{p(x_t\mid x_{<t})^\alpha}{Z_t(x_{<t})}=\frac{p(\bx)^\alpha}{\prod_{t=1}^{T}Z_t(x_{<t})}.
\]
The sequence-level power distribution instead uses a single global normalizer, $p^\alpha(\bx)=p(\bx)^\alpha/Z$ with $Z=\sum_{\bx}p(\bx)^\alpha$. These coincide only if $\prod_{t}Z_t(x_{<t})$ is the same constant $Z$ for every $\bx$ in the support, which fails whenever the per-step normalizers $Z_t(x_{<t})$ depend on the prefix $x_{<t}$---the generic case for a non-uniform model. Thus the two distributions differ in general.
\end{proof}

\subsection{Why collapse depth concentrates on hard tokens}

\begin{proposition}[Concentrated futures collapse early]\label{prop:concentrated-futures}
Fix a token position $t$ and consider its logit-lens distributions $p_l(\cdot\mid x_{\le t})$ across layers $l$. Suppose at some layer $l_0$ the distribution is already concentrated, $\max_v p_{l_0}(v\mid x_{\le t})\ge 1-\epsilon$ for small $\epsilon\in(0,1)$. Then the logit-lens entropy at $l_0$ satisfies $H_{l_0}(t)\le H_b(\epsilon)+\epsilon\log(|\mathcal{X}|-1)$, where $H_b$ is the binary entropy function. In particular, for small $\epsilon$ the entropy is below a modest threshold $\tau$, so $d_t\le l_0$.
\end{proposition}

\begin{proof}
This is Fano's inequality applied to the logit-lens distribution. Let $v^\star=\argmax_v p_{l_0}(v\mid x_{\le t})$ with mass $1-\delta$ for some $\delta\le\epsilon$. Writing the entropy as that of a mixture between $v^\star$ and the remaining mass,
\[
H_{l_0}(t) = H_b(\delta) + \delta \cdot H\!\big(p_{l_0}(\cdot\mid x_{\le t},\, v\neq v^\star)\big) \le H_b(\delta) + \delta\log(|\mathcal{X}|-1),
\]
since the conditional distribution over the other $|\mathcal{X}|-1$ tokens has entropy at most $\log(|\mathcal{X}|-1)$. Both $H_b(\delta)$ and $\delta\log(|\mathcal{X}|-1)$ are nondecreasing in $\delta$ on $[0,1/2]$, so the bound holds with $\delta$ replaced by $\epsilon$. For $\epsilon$ small enough that $H_b(\epsilon)+\epsilon\log(|\mathcal{X}|-1)\le\tau$, we get $H_{l_0}(t)\le\tau$, hence $d_t=\min\{l\in\Lsub:H_l(t)\le\tau\}\le l_0$.
\end{proof}

\subsection{Reranking approximates the DEGS target}

\begin{proposition}[Self-normalized importance sampling]\label{prop:reranking-approx}
Let $\{\bx^{(i)}\}_{i=1}^N$ be drawn i.i.d.\ from a proposal $g$ with $g(\bx)>0$ on $\mathcal{S}$. Define importance weights $w_i\propto \pi(\bx^{(i)})/g(\bx^{(i)})$, normalized to sum to one. Then for any bounded test function $h$, the self-normalized estimator $\sum_i w_i h(\bx^{(i)})$ converges almost surely to $\mathbb{E}_{\pi}[h]$ as $N\to\infty$. Choosing $g=p$ gives unnormalized weights $w_i\propto p(\bx^{(i)})^{\alpha-1}\exp(\beta D(\bx^{(i)}))$.
\end{proposition}

\begin{proof}
Self-normalized importance sampling is consistent whenever $\mathrm{supp}(\pi)\subseteq\mathrm{supp}(g)$~\citep{neal1993probabilistic}. The unnormalized weight is $\pi(\bx)/g(\bx)\propto p(\bx)^\alpha\exp(\beta\Dx)/p(\bx)=p(\bx)^{\alpha-1}\exp(\beta\Dx)$ when $g=p$, and the global constant $Z$ cancels under normalization. Almost-sure convergence follows from the strong law of large numbers applied to numerator and denominator separately, both of which have finite expectation since $h$ is bounded and the weights are integrable on the finite support $\mathcal{S}$. DEGS reranking returns the maximizer $\argmax_i [\alpha\log p(\bx^{(i)})+\beta D(\bx^{(i)})]$, the mode of the reweighted set; the same weights underlie the full self-normalized estimator.
\end{proof}

\subsection{Cost accounting}

\begin{proposition}[Probing overhead]\label{prop:overhead}
Let a full forward pass over a length-$T$ sequence cost $C_{\mathrm{fwd}}$ (dominated by $L$ transformer layers). Computing $\Dx$ by caching hidden states at $K=|\Lsub|$ layers and applying the logit lens adds entropy computations at those $K$ layers only. The marginal cost of probing is $C_{\mathrm{probe}}=K\cdot c_{\mathrm{lens}}$, where $c_{\mathrm{lens}}$ is the per-layer logit-lens-plus-entropy cost, independent of $L$; relative to a generation forward pass that already computes all $L$ layers, probing reuses the same hidden states and adds only the lens/entropy arithmetic.
\end{proposition}

\begin{proof}
The forward pass computes hidden states $\mathbf{h}_{l,t}$ at all layers regardless of probing. Probing reads the cached states at the $K$ layers in $\Lsub$ and, for each, applies the unembedding and computes entropy, costing $c_{\mathrm{lens}}$ per layer (an $O(|\mathcal{X}|d)$ matrix-vector product plus an $O(|\mathcal{X}|)$ entropy sum, optionally chunked over the vocabulary). Summing over the $K$ probed layers gives $C_{\mathrm{probe}}=K\,c_{\mathrm{lens}}$, with no dependence on the total depth $L$ beyond $K\le L$. The teacher-forced scoring pass is a single additional forward, so the total scoring cost is $C_{\mathrm{fwd}}+C_{\mathrm{probe}}$.
\end{proof}

\begin{proposition}[Token budget of DEGS-MCMC]\label{prop:token-budget}
With block size $B$, sequence length $T$, and $N_{\mathrm{MCMC}}$ Metropolis--Hastings steps per block, the expected number of generated tokens is $\Theta(N_{\mathrm{MCMC}} T^2 / B)$ up to the per-step resampling fraction, matching the accounting of \citet{karan2025reasoning}.
\end{proposition}

\begin{proof}
There are $T/B$ blocks. Within block $k$, the state has length $\Theta(kB)$, and a random-index proposal resamples on average half the suffix from the chosen cut point, i.e.\ $\Theta(kB)$ tokens per accepted/attempted proposal. Running $N_{\mathrm{MCMC}}$ steps per block and summing over $k=1,\dots,T/B$ gives $\sum_{k=1}^{T/B} N_{\mathrm{MCMC}}\,\Theta(kB)=N_{\mathrm{MCMC}}\,\Theta\!\big(B\cdot \tfrac{(T/B)^2}{2}\big)=\Theta(N_{\mathrm{MCMC}} T^2/B)$. This reproduces the quadratic-in-$T$ cost of block-wise power sampling.
\end{proof}

\subsection{Entropy decomposition for the depth signal}

\begin{definition}[Layer-conditional entropy contribution]\label{def:layer-kl}
For token $t$, define the inter-layer change in logit-lens belief by the KL divergence $\Delta_l(t)=\mathrm{KL}\!\big(p_{l}(\cdot\mid x_{\le t})\,\|\,p_{l-1}(\cdot\mid x_{\le t})\big)$ between consecutive layers $l-1$ and $l$.
\end{definition}

\begin{proposition}[Collapse depth tracks belief stabilization]\label{prop:entropy-decomposition}
If the logit-lens beliefs stabilize after layer $l^\star$ in the sense that $\Delta_l(t)\le\eta$ for all $l>l^\star$ and some small $\eta$, and the entropy $H_l(t)$ is non-increasing for $l>l^\star$, then the per-token collapse depth satisfies $d_t\le l^\star$ whenever $H_{l^\star}(t)\le\tau$. Consequently $\Dx$ is governed by the depth at which beliefs stabilize, not by post-stabilization fluctuations.
\end{proposition}

\begin{proof}
By definition $d_t=\min\{l\in\Lsub:H_l(t)\le\tau\}$. If $H_{l^\star}(t)\le\tau$ then the minimizing layer is at most $l^\star$, so $d_t\le l^\star$. The stabilization condition $\Delta_l(t)\le\eta$ for $l>l^\star$ ensures the beliefs (and hence the entropy) change negligibly beyond $l^\star$, so the threshold crossing that defines $d_t$ is determined at or before $l^\star$; later near-constant layers cannot move the crossing earlier or later. Averaging $d_t/L$ over tokens, $\Dx$ depends on the stabilization depths $\{l^\star_t\}$ rather than on fluctuations in deeper layers.
\end{proof}

\subsection{Bounded perturbation of the likelihood ranking}

\begin{proposition}[Depth term is a bounded reranking perturbation]\label{prop:bounded-perturbation}
Fix two candidate sequences $\bx,\bx'$. Under the DEGS score $S(\bx)=\alpha\log p(\bx)+\beta\Dx$ with $\Dx\in(0,1]$, the depth term can change the score gap by at most $\beta$ in absolute value: $\big|[S(\bx')-S(\bx)]-\alpha[\log p(\bx')-\log p(\bx)]\big|=\beta\,|D(\bx')-\Dx|\le\beta$. Hence if $\alpha|\log p(\bx')-\log p(\bx)|>\beta$, the DEGS ranking of $\bx,\bx'$ agrees with the pure-likelihood ranking; the depth signal can only reorder candidates whose $\alpha$-scaled log-likelihoods lie within $\beta$ of each other.
\end{proposition}

\begin{proof}
By definition $S(\bx')-S(\bx)=\alpha[\log p(\bx')-\log p(\bx)]+\beta[D(\bx')-\Dx]$, so the deviation from the pure-likelihood gap is exactly $\beta[D(\bx')-\Dx]$. Since $\Dx,D(\bx')\in(0,1]$, their difference lies in $(-1,1)$, giving the bound $\beta|D(\bx')-\Dx|\le\beta$. If $\alpha[\log p(\bx')-\log p(\bx)]>\beta$, then $S(\bx')-S(\bx)\ge \alpha[\log p(\bx')-\log p(\bx)]-\beta>0$, preserving the likelihood order; symmetrically for the reverse inequality. Thus reordering requires $\alpha|\log p(\bx')-\log p(\bx)|\le\beta$.
\end{proof}

\begin{remark}[Interpretation]\label{rem:interpretation}
Proposition~\ref{prop:bounded-perturbation} formalizes the design intent of DEGS: the depth-entropy term acts as a \emph{tiebreaker} among candidates of comparable likelihood rather than overriding the base model's likelihood signal. With the default $\beta=5$ and $\alpha=4$, two candidates are reordered only if their log-likelihoods differ by at most $\beta/\alpha=1.25$ nats---a small gap on the scale of full-sequence log-likelihoods---so DEGS preserves the strong likelihood ranking almost everywhere and intervenes only where the base model is itself nearly indifferent. This is consistent with the weak per-candidate AUC of collapse depth (Section~\ref{sec:motivation}) coexisting with a consistent decoding gain: a near-chance signal, applied only at genuine likelihood ties and compounded over many MCMC steps, shifts the sampled distribution without destabilizing it.
\end{remark}

\begin{remark}[Scope of the theory]\label{rem:scope}
The results above are deliberately modest: Propositions~\ref{prop:normalizable}--\ref{prop:token-budget} establish that the DEGS target is well-defined, that DEGS-MCMC targets it correctly, and that the probing overhead is depth-independent, while Proposition~\ref{prop:bounded-perturbation} bounds the influence of the depth term. None of these results claims that larger $\Dx$ \emph{causes} higher correctness; that link is the empirical hypothesis tested in Section~\ref{sec:experiments} (Figure~\ref{fig:collapse-depth-predicts-correctness}), and the theory only delimits how a weak correctness signal, once assumed, is incorporated into sampling without distorting the likelihood objective beyond a controlled margin.
\end{remark}

\end{document}